\documentclass[runningheads]{llncs}
\usepackage{graphicx}
\usepackage{adjustbox}

\usepackage{geometry}
\usepackage{amsmath,amssymb}
\usepackage{color,colortbl}
\usepackage[dvipsnames]{xcolor}
\definecolor{cycle2}{RGB}{55, 126, 184}

\usepackage{algorithmic}% format pseudocode
\usepackage[vlined,boxed,commentsnumbered, algo2e]{algorithm2e}
\usepackage{algorithm}% float wrapper for algorithms

\usepackage{booktabs} % toprule
\usepackage{amssymb} % For the checkmark
\usepackage{hyperref} % links
\usepackage{bbm}
\usepackage{url}            % simple URL typesetting
\usepackage{amsfonts}       % blackboard math symbols
\usepackage{nicefrac}       % compact symbols for 1/2, etc.
\usepackage{microtype}      % microtypography
\usepackage{graphicx}
\usepackage{subfigure}
\usepackage{setspace}
\usepackage{multirow}
\usepackage{cite}

\usepackage{amsmath}

\usepackage{array}

\newcolumntype{P}[1]{>{\centering\arraybackslash}p{#1}}
\newcommand{\pluseq}{\mathrel{+}=}

\newtheorem{axiom}{Axiom}

\begin{document}

\title{GraphSVX: Shapley Value Explanations for Graph Neural Networks}


\author{Alexandre Duval \and
Fragkiskos D. Malliaros}

\institute{Université Paris-Saclay, CentraleSupélec, Inria, Gif-Sur-Yvette, France \\
\email{alexandre.duval@centralesupelec.fr} \\
\email{fragkiskos.malliaros@centralesupelec.fr}}

\maketitle              % typeset the header of the contribution

\begin{abstract}
Graph Neural Networks (GNNs) achieve significant performance for various learning tasks on geometric data due to the incorporation of graph structure into the learning of node representations, which renders their comprehension challenging. In this paper, we first propose a unified framework satisfied by most existing GNN explainers. Then, we introduce GraphSVX, a post hoc local model-agnostic explanation method specifically designed for GNNs. GraphSVX is a decomposition technique that captures the ``fair'' contribution of each feature and node towards the explained prediction by constructing a surrogate model on a perturbed dataset. It extends to graphs and ultimately provides as explanation the Shapley Values from game theory. Experiments on real-world and synthetic datasets demonstrate that GraphSVX achieves state-of-the-art performance compared to baseline models while presenting core theoretical and human-centric properties. 
% \keywords{First keyword  \and Second keyword \and Another keyword.}
\end{abstract}

\section{Introduction}
Many aspects of the everyday life involve data without regular spatial structure, known as non-euclidean or geometric data, such as social networks, molecular structures or citation networks \cite{cho2011friendship, backstrom2011supervised, duvenaud2015convolutional}. These datasets, often represented as graphs, are challenging to work with because they require modelling rich relational information on top of node feature information \cite{zhou2018graph, zhang2020deep}. Graph Neural Networks (GNNs) are powerful tools for representation learning of such data. They achieve state-of-the-art performance on a wide variety of tasks \cite{defferrard2016convolutional, xu2019spatio, zhang2018link} due to their recursive message passing scheme, where they encode information from nodes and pass it along the edges of the graph. Similarly to traditional deep learning frameworks, GNNs showcase a complex functioning that is rather opaque to humans. As the field grows, understanding them becomes essential for well known reasons, such as ensuring privacy, fairness, efficiency, and safety \cite{o2016weapons, kim2015interactive, duval2019explainable}.

\par While there exist a variety of explanation methods \cite{saltelli2002sensitivity, simonyan2013deep, goldstein2015peeking, shrikumar2017learning, selvaraju2017grad, zhang2018top}, they are not well suited for geometric data as they fall short in their ability to incorporate graph topology information. \cite{baldassarre2019explainability, pope2019explainability} have proposed extensions to GNNs, but in addition to limited performance, they require model internal knowledge and show gradient saturation issues due to the discrete nature of the adjacency matrix. 
\par GNNExplainer \cite{ying2019gnnexplainer} is the first explanation method designed specifically for GNNs. It learns a continuous (and a discrete) mask over the edges (and features) of the graph by formulating an optimisation process that maximizes mutual information between the distribution of possible subgraphs and GNN prediction. More recently, PGExplainer \cite{luo2020parameterized} and GraphMask \cite{schlichtkrull2020interpreting} generalize GNNExplainer to an inductive setting; they use re-parametrisation tricks to alleviate the ``introduced evidence" problem \cite{dabkowski2017real}--- i.e. continuous masks deform the adjacency matrix and introduce new semantics to the generated graph. Regarding other approaches; GraphLIME \cite{huang2020graphlime} builds on LIME \cite{ribeiro2016should} to provide a non-linear explanation model; PGM-Explainer \cite{vu2020pgm} learns a simple Bayesian network handling node dependencies; XGNN \cite{yuan2020xgnn} produces model-level insights via graph generation trained using reinforcement learning.

\par Despite recent progress, existing explanation methods do not relate much and show clear limitations. Apart from GNNExplainer, none considers node features together with graph structure in explanations. Besides, they do not present core properties of a ``good" explainer \cite{molnar2020interpretable} (see Sec. \ref{_Rel}). Since the field is very recent and largely unexplored, there is little certified knowledge about explainers' characteristics. It is, for instance, unclear whether optimising mutual information is pertinent or not. Overall, this often yields explanations with a poor signification, like a probability score stating how essential a variable is \cite{ying2019gnnexplainer, luo2020parameterized, schlichtkrull2020interpreting}. Existing techniques not only lack strong theoretical grounds, but also do not showcase an evaluation that is sophisticated enough to properly justify their effectiveness or other desirable aspects \cite{robnik2018perturbation}. Lastly, little importance is granted to their human-centric characteristics \cite{miller2017explainable}, limiting the comprehensibility of explanations from a human perspective. 

\par In light of these limitations, first, we propose a unified explanation framework encapsulating recently introduced explainers for GNNs. It not only serves as a connecting force between them but also provides a different and common view of their functioning, which should inspire future work. In this paper, we exploit it ourselves to define and endow our explainer, GraphSVX, with desirable properties.
More precisely, GraphSVX carefully constructs and combines the key components of the unified pipeline so as to jointly capture the average marginal contribution of node features and graph nodes towards the explained prediction. We show that GraphSVX ultimately computes, via an efficient algorithm, the \textit{Shapley values} from game theory \cite{shapley1953value}, that we extend to graphs. The resulting unique explanation thus satisfy several theoretical properties by definition, while it is made more human-centric through several extensions. In the end, we evaluate GraphSVX on real-world and synthetic datasets for node and graph classification tasks. We show that it outperforms existing baselines in explanation accuracy, and verifies further desirable aspects such as robustness or certainty. \\

\noindent \textbf{Source code.} The source code is available at \textbf{\url{https://github.com/AlexDuvalinho/GraphSVX}}.

\label{_Intro}

\section{Related Work}
Explanations methods specific to GNNs are classified into five categories of methods according to \cite{yuan2020explainability}: gradient-based, perturbation, decomposition, surrogate, and model-level. We utilise the same taxonomy in this paper to position GraphSVX. \\

\noindent
\textbf{Decomposition methods} \cite{baldassarre2019explainability, pope2019explainability} distribute the prediction score among input features using the weights of the network architecture, through backpropagation.
% They first decompose it to the neurons in the last hidden layer, and backpropagate iteratively these importance scores until input neurons. 
Despite offering a nice interpretation, they are not specific to GNNs and present several major limits such as requiring access to model parameters or being sensitive to small input changes, like \textbf{gradient-based methods} discussed in Sec. \ref{_Intro}.  \\

\noindent
\textbf{Perturbation methods} \cite{ying2019gnnexplainer, luo2020parameterized, schlichtkrull2020interpreting} monitor variations in model prediction with respect to different input perturbations. Such methods provide as explanation a continuous mask over edges (features) holding importance probabilities learned via a simple optimisation procedure, affected by the introduced-evidence problem. \\
% The resulting explanations only show an importance probability and thus have little signification. 

\noindent
\textbf{Surrogate methods} \cite{huang2020graphlime, vu2020pgm} 
approximate the black box GNN model locally by learning an interpretable model on a dataset built around the instance of interest $v$ (e.g., neighbours). Explanations for the surrogate model are used as explanations for the original model. For now, such approaches are rather intuition-based and consider exclusively node features or graph topology, not both. \\

\noindent
\textbf{Model level methods} \cite{yuan2020xgnn} provide general insights on model functioning. It supports only graph classification, requires an input candidate node set and is challenged by local methods also giving global explanations \cite{luo2020parameterized}. \\

\noindent
As we will show shortly, GraphSVX bridges the gap between these categories by learning a surrogate explanation model on a perturbed dataset that ultimately decomposes the explained prediction among the nodes and features of the graph, depending on their respective contribution. It also derives model-level insights by explaining subsets of nodes, while avoiding the respective limits of each category. \\

\noindent
\textbf{Desirable properties of explanations} have received subsequent attention from the social sciences and the machine learning communities, but are often overlooked when designing an explainer. From a theoretical perspective, good explanations are accurate, fidel (\textit{truthful}), and reflect the proportional importance of a feature on prediction (\textit{meaningful}) \cite{burkart2021survey, yuan2020explainability}. They also are stable and consistent (\textit{robust}), meaning with a low variance when changing to a similar model or a similar instance \cite{molnar2020interpretable}. Besides, they reflect the certainty of the model (\textit{decomposable}) and are as representative as possible of its (\textit{global}) functioning \cite{miller2019explanation}. 
Finally, since their ultimate goal is to help humans understand the model, explanations should be intuitive to comprehend (\textit{human-centric}). Many sociological and psychological studies emphasise key aspects: only a few motives (\textit{selective}) \cite{ustun2014methods}, comparable to other instances (\textit{contrastive}) \cite{lipton1990contrastive}, and interactive with the explainee (\textit{social}). We refer to Appendix \ref{A_prop} for more rigorous definitions and to see how GNN explainers satisfy them.  
\label{_Rel}

\section{Preliminary Concepts and Background}
\textbf{Notation.} We consider a graph $\mathcal{G}$ with $N$ nodes and $F$ features defined by $(\mathbf{X},\mathbf{A})$ where $\mathbf{X} \in \mathbb{R}^{N \times F}$ is the feature matrix and $\mathbf{A} \in \mathbb{R}^{N \times N}$ the adjacency matrix. $f(\mathbf{X}, \mathbf{A})$ denotes the prediction of the GNN model $f$, and $f_v(\mathbf{X},\mathbf{A})$ the score of the predicted class for node $v$. % (or the one we wish to explain) 
Let $\mathbf{X}_{*j} = (X_{1j},\ldots,X_{Nj})$ with feature values $\mathbf{x}_{*j}=(x_{1j},\ldots,x_{Nj})$ represent feature $j$'s value vector across all nodes. Similarly, $\mathbf{X}_{i} = \mathbf{X}_{i*} = (X_{i1},\ldots,X_{iF})$ stands for node $i$'s feature vector, with  $\mathbf{X}_{iS}=\{X_{ik} | k \in S\}$. $\mathbf{1}$ is the all-ones vector.

\subsection{Graph Neural Networks}
GNNs adopt a message passing mechanism \cite{hamilton2017inductive} where the update at each GNN layer $\ell$ involves three key calculations \cite{battaglia2018relational}:
%\cite{kipf2016semi, xu2018powerful,  velivckovic2017graph} 
(i) The propagation step. The model computes a message $m_{ij}^\ell = \textsc{Msg}(\mathbf{h}_i^{\ell-1},\mathbf{h}_j^{\ell-1},a_{ij})$ between every pair of nodes $(v_i, v_j)$, that is, a function $\textsc{MSG}$ of $v_i$'s and $v_j$'s representations $\mathbf{h}_i^{\ell-1}$ and $\mathbf{h}_j^{\ell-1}$ in the previous layer and of the relation $a_{ij}$ between the nodes. (ii) The aggregation step. For each node $v_i$, GNN calculates an aggregated message $M_i$ from $v_i$'s neighbourhood $\mathcal{N}_{v_i}$, whose definition vary across methods. $M_i^\ell = \textsc{Agg}(m_{ij}^\ell|v_j \in \mathcal{N}_{v_i})$. (iii) The update step. GNN non-linearly transforms both the aggregated message $M_i^\ell$ and $v_i$'s representation $\mathbf{h}_i^{\ell-1}$ from the previous layer, to obtain $v_i$'s representation $\mathbf{h}_i^\ell $ at layer $\ell$: $\mathbf{h}_i^\ell = \textsc{Upd}(M_i^\ell,\mathbf{h}_i^{\ell-1})$. The 
representation $\mathbf{z}_i = \mathbf{h}_i^L$ of the final GNN layer $L$ serves as final node embedding and is used for downstream machine learning tasks.

\subsection{The Shapley value}

The Shapley value is a method from Game Theory. It describes how to fairly distribute the total gains of a game to the players depending on their respective contribution, assuming they all collaborate. 
It is obtained by computing the average marginal contribution of each player when added to any possible coalition of players \cite{shapley1953value}. This method has been extended to explain machine learning model predictions on tabular data \cite{lipovetsky2001analysis, strumbelj2010efficient}, assuming that each feature of the explained instance ($\mathbf{x}$) is a player in a game where the prediction is the payout.
% winter2002shapley

\par The characteristic function $\text{\textit{val}}: S \rightarrow \mathbb{R}$ captures the marginal contribution of the coalition $S \subseteq \{1,\ldots, F\}$  of features towards the prediction $f(\mathbf{x})$ with respect to the average prediction: $\text{\textit{val}}(S)= \mathbb{E}[f(\mathbf{X})|\mathbf{X}_S=\mathbf{x}_s] - \mathbb{E}[f(\mathbf{X})]$. We isolate the effect of a feature $j$ via $\text{\textit{val}}(S\cup \{j\}) - \text{\textit{val}}(S)$ and average it over all possible ordered coalitions $S$ to obtain its Shapley value as: 
\begin{align*}
    \phi _{j}(\text{\textit{val}}) = \sum _ {S \subseteq \{1,\ldots,F \} \setminus \{j  \} } 
    \dfrac{\vert S \vert ! ~ (F- \vert S \vert - 1)!}{F!} 
\big(\text{\textit{val}}(S \cup \{ j \}) - \text{\textit{val}}(S)\big).
\end{align*}
The notion of fairness is defined by four axioms (\textit{efficiency}, \textit{dummy}, \textit{symmetry}, \textit{additivity}), and the Shapley value is the unique solution satisfying them. In practice, the sum becomes impossible to compute because the number of possible coalitions ($2^{F-1}$) increases exponentially by adding more features. We thus approximate Shapley values using sampling \cite{vstrumbelj2014explaining, datta2016algorithmic, lundberg2017unified}.

\label{_PK}

\section{A Unified Framework for GNN Explainers}
\begin{figure*}[t]
    \centering
    %\vspace{-.1cm}
    \includegraphics[width= \textwidth]{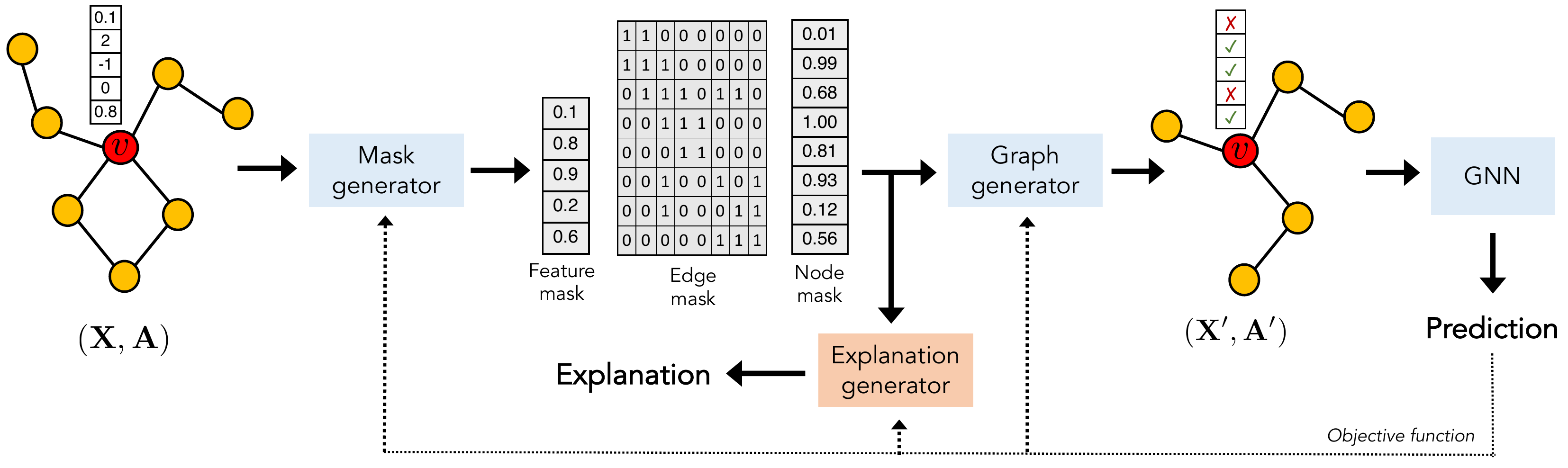}
    %\vspace{-.1cm}
    \caption{Overview of unified framework. All methods take as input a given graph $\mathcal{G}=(\mathbf{X},\mathbf{A})$, feed it to a mask generator (\textsc{Mask}) to create three masks over nodes, edges and features. These masks are then passed to a graph generator (\textsc{Gen}) that converts them to the original input space $(\mathbf{X}',\mathbf{A}')$ before feeding them to the original GNN model $f$. The resulting prediction $f(\mathbf{X}',\mathbf{A}')$ is used to improve the mask generator, the graph generator or the downstream explanation generator (\textsc{Expl}), which ultimately provides the desired explanation--using masks and $f(\mathbf{X}',\mathbf{A}')$. This passage through the framework is repeated many times to create a proper dataset $\mathcal{D}$ from which each generator block learns. Usually, only one is optimised with a carefully defined process involving the new and original GNN predictions.}
    \label{unif_pipeline}
\end{figure*}

As detailed in the previous section, existing interpretation methods for GNNs are categorised and often treated separately. In this paper, we approach the explanation problem from a new angle, proposing a unified view that regroups existing explainers under a single framework: GNNExplainer, PGExplainer, GraphLIME, PGM-Explainer, XGNN, and the proposed GraphSVX. The key differences across models lie in the definition and optimisation of the three main blocks of the pipeline, as shown in Fig. \ref{unif_pipeline}:
\begin{itemize}
    \item \textsc{Mask} generates discrete or continuous masks over features $\mathbf{M}_F \in \mathbb{R}^{F} $, nodes $\mathbf{M}_N \in \mathbb{R}^{N}$ and edges $\mathbf{M}_E \in \mathbb{R}^{N \times N}$, according to a specific strategy. 
    \item \textsc{Gen} outputs a new graph $\mathcal{G}'=(\mathbf{X}',\mathbf{A}')$ from the masks $(\mathbf{M}_E, \mathbf{M}_N, \mathbf{M}_F)$ and the original graph $\mathcal{G}=(\mathbf{X},\mathbf{A})$.
    \item \textsc{Expl} generates explanations, often offered as a vector or a graph, using a function $g$ whose definition vary across baselines.
\end{itemize}

In the following, we show how each baseline fits the pipeline. $\odot$ stands for the element wise multiplication operation, $\sigma$ the softmax function, $||$ the concatenation operation, and $\textbf{M}^{\text{ext}}$ describes the extended vector $\mathbf{M}$ with repeated entries, whose size makes the operation feasible. All three masks are not considered for a single method; some are ignored as they have no effect on final explanations. Indeed, one often studies node feature $\mathbf{M}_F$ or graph structure (via $\mathbf{M}_E$ or $\mathbf{M}_N$). \\

\noindent
\textbf{GNNExplainer}'s key component is the mask generator. It generates both $\mathbf{M}_F$ and $\mathbf{M}_E$, where $\mathbf{M}_E$ has continuous values and $\mathbf{M}_F$ discrete ones. They are both randomly initialised and jointly optimised via a mutual information loss function  $MI(Y, (\mathbf{M}_E,\mathbf{M}_F)) = H(Y) - H(Y|\mathbf{A}', \mathbf{X}')$, where $\textsc{Gen}$ gives $\mathbf{A}'= \mathbf{A} \odot \sigma(\mathbf{M}_E)$ and $\mathbf{X}' = \mathbf{X} \odot \mathbf{M}_F^{\text{ext}}$. $Y$ represents the class label and $H(\cdot)$ the entropy term. \textsc{Expl} simply returns the learned masks as explanations, via the identity function $g(\mathbf{M}_E, \mathbf{M}_F)= (\mathbf{M}_E, \mathbf{M}_F)$. \\
% (or $\mathbf{X}' = \mathbf{Z} + (\mathbf{X}-\mathbf{Z}) \odot \mathbf{M}_F^{ext}$ with $\mathbf{Z}$ a random variable sampled from the dataset empirical distribution st $\sum \mathbf{M}_F_j \leq c$

\noindent
\textbf{PGExplainer} is very similar to GNNExplainer. $\textsc{Mask}$ generates only an edge mask $\mathbf{M}_E$ using a multi-layer neural network $\textsc{MLP}_{\psi}$ and the learned matrix $\mathbf{Z}$ of node representations: $\mathbf{M}_E = \textsc{MLP}_{\psi}(\mathcal{G},\mathbf{Z})$. The new graph is constructed with $\textsc{Gen}(\mathbf{X},\mathbf{A},\mathbf{M}_E) =
(\mathbf{X}, \mathbf{A} \odot \rho(\mathbf{M}_E))$, where $\rho$ denotes a reparametrisation trick. The obtained prediction $f_v(\mathbf{X},\mathbf{A}')$ is also used to maximise mutual information with $f_v(\mathbf{X},\mathbf{A})$ and backpropagates the result to optimise $\textsc{Mask}$. As for GNNExplainer, $\textsc{Expl}$ provides $\mathbf{M}_E$ as explanations. \\

\noindent
\textbf{GraphLIME} is a surrogate method with a simple and not optimised mask generator. Although it measures feature importance, it creates a node mask $\mathbf{M}_N$ using the neighbourhood of $v$ (i.e., $\mathcal{N}_v$). The $k^{th}$ mask (or sample) is defined as $\mathbf{M}_{N,i}^k = 1 ~~ \text{ if } v_i = \mathcal{N}_v[k] \text{ and } 0$ otherwise. 
% $ 
% \mathbf{M}_{N_i}^k = 
% \begin{cases} 
% 1, ~~ \text{ if } v_i = \mathcal{N}(v)[k] \\
% 0, ~~ \text{ otherwise.}
% \end{cases}
% $
\textsc{Gen}$(\mathbf{X},\mathbf{A},\mathbf{M}_N) = (\mathbf{X},\mathbf{A})$, so in fact, it computes and stores the original model prediction. $\mathbf{X}$ and $f(\mathbf{X},\mathbf{A})$ are then combined with the mask $\mathbf{M}_N$ via simple dot products $\mathbf{M}_N^\top \cdot \mathbf{X}$ and $\mathbf{M}_N^\top \cdot f(\mathbf{X},\mathbf{A})$ respectively, to isolate the original feature vector and prediction of the $k^{th}$ neighbour of $v$. These two elements are treated as input and target of an HSIC Lasso model $g$, trained with an adapted loss function. The learned coefficients constitute importance measures that are given as explanations by $\textsc{Expl}$. \\

\noindent
\textbf{PGM-Explainer} builds a probabilistic graphical model on a local dataset that consists of random node masks $\mathbf{M}_N \in \{0,1\}^N$. The associated prediction $f_v(\mathbf{X}',\mathbf{A}')$ is obtained by posing $\mathbf{A'} = \mathbf{A}$ and $\mathbf{X}'= \mathbf{M}_N^{\text{ext}} \odot \mathbf{X} + (\mathbf{1} - \mathbf{M}_N^{\text{ext}} \odot \boldsymbol{\mu}^{\text{ext}})$, with $\boldsymbol{\mu} = (E[\mathbf{X}_{*1}], \ldots, E[\mathbf{X}_{*F}])^\top$. This means that each excluded node feature ($\mathbf{M}_{N,j}=0$) is set to its mean value across all nodes. This dataset is fed sequentially to the main component $\textsc{Expl}$, which learns and outputs a Bayesian Network $g$ with input $\mathbf{M}_N$ (made sparser by looking at the Markov-blanket of $v$), \textit{BIC score} loss function, and target $I(f_v(\mathbf{X}',\mathbf{A}'))$, where $I(\cdot)$ is a specific function that quantifies the difference in prediction between original and new prediction. \\

\noindent
\textbf{XGNN} is a model-level approach that trains an iterative graph generator (add one edge at a time) via reinforcement learning. This causes two key differences with previous approaches: (1) the input graph at iteration $t$ ($\mathcal{G}_t$) is obtained from the previous iteration and is initialised as the empty graph;  (2) we also pass a candidate node set $\mathcal{C}$, such that $\mathbf{X}_\mathcal{C}$ contains the feature vector of all distinct nodes across all graphs in dataset. $\textsc{Mask}$ generates an edge mask $\mathbf{M}_E=\mathbf{A}_t$ and a node mask $\mathbf{M}_{N_t} \in \{0,1\}^{|\mathcal{C}|}$ specifying the latest node added to $\mathcal{G}_t$, if any. $\textsc{Gen}$ produces a new graph $\mathcal{G}_{t+1}$ from $\mathcal{G}_{t}$ by predicting a new edge, possibly creating a new node from $\mathcal{C}$. This is achieved by applying a GCN and two MLP networks. 
% $\mathbf{\hat{X}} = \textsc{GCN}(\mathcal{C},\mathbf{X}, \mathbf{A} )$
% $\sigma(\textsc{MLP}_1(\mathbf{\hat{X}})) \cdot \mathbf{M}_N$, $\sigma(\textsc{MLP}_2(\mathbf{\hat{X}},\mathbf{\hat{X}}_{start}))$.
Then, $\mathcal{G}_{t+1}$ is fed to the explained GNN. The resulting prediction is used to update model parameters via a policy gradient loss function. $\textsc{Expl}$ stores nonzero $\mathbf{M}_{N_t}$ at each time step and provides $g(\{\mathbf{M}_{N_t}\}_t, \mathbf{X}_C, \mathbf{M}_E)) = (||_t \mathbf{M}_{N_t} \cdot \mathbf{X}_C , \mathbf{M}_E)$ as explanation --- i.e. the graph generated at the final iteration, written $\mathcal{G}_{T}$. \\
% $\mathbf{M}_N$ specifies which node types are contained and how they are ordered in the new graph and $\mathbf{M}_E$ specifies the relation between them. 

\noindent
\textbf{GraphSVX}. 
As we will see in the next section, the proposed GraphSVX model carefully exploits the potential of this framework through a better design and combination of complex mask, graph and explanation generators--in the perspective of improving performance and embedding desirable properties in explanations.

\label{_Uni}

\section{Proposed Method}
GraphSVX is a post hoc model-agnostic explanation method specifically designed for GNNs, that jointly computes graph structure and node feature explanations for a single instance. More precisely, GraphSVX constructs a perturbed dataset made of binary masks for nodes and features $(\mathbf{M}_N, \mathbf{M}_F)$, and computes their marginal contribution $f(\mathbf{X'},\mathbf{A'})$ towards the prediction using a graph generator $\textsc{Gen}(\mathbf{X},\mathbf{A}, \mathbf{M}_F,\mathbf{M}_N) = (\mathbf{X'},\mathbf{A'})$. It then learns a carefully defined explanation model on the dataset $( \mathbf{M}_N || \mathbf{M}_F,f(\mathbf{X'},\mathbf{A'}) )$ and provides it as explanation. Ultimately, it produces a unique deterministic explanation that decomposes the original prediction and has a real signification (Shapley values) as well as other desirable properties evoked in Sec. \ref{_Rel}. Without loss of generality, we consider a node classification task for the presentation of the method.

\subsection{Mask and graph generators}
\label{MGG}

First of all, we create an efficient mask generator algorithm that constructs discrete feature and node masks, respectively denoted by $\mathbf{M}_F \in \{0,1\}^F$ and $\mathbf{M}_N \in \{0,1\}^N$. 
Intuitively, for the explained instance $v$, we aim at studying the joint influence of a subset of features and neighbours of $v$ towards the associated prediction $f_v(\mathbf{X},\mathbf{A})$. The mask generator helps us determine the subset being studied. Associating $1$ with a variable (node or feature) means that it is considered, $0$ that it is discarded. 
% Only k-hops neighbours and features different from expected value here
For now, we let $\textsc{Mask}$ randomly sample from all possible ($2^{F+N-1}$) pairs of masks $\mathbf{M}_F$ and $\mathbf{M}_N$, meaning all possible coalitions $S$ of  features and nodes ($v$ is not considered in explanations). Let $\mathbf{z}$ be the random variable accounting for selected variables, $\mathbf{z}=(\mathbf{M}_F \| \mathbf{M}_N)$. This is a simplified version of the true mask generator, which we will come back to later, in Sec. \ref{efficient_approx}. 

\par We now would like to estimate the joint effect of this group of variables towards the original prediction. We thus isolate the effect of selected variables marginalised over excluded ones, and observe the change in prediction. We define $\textsc{Gen}: (\mathbf{X},\mathbf{A}, \mathbf{M}_F,\mathbf{M}_N) \rightarrow (\mathbf{X}',\mathbf{A'})$, which converts the obtained masks to the original input space, in this perspective. 
Due to the message passing scheme of GNNs, studying jointly node and features' influence is tricky. Unlike GNNExplainer, we avoid any overlapping effect by considering feature values of $v$ (instead of the whole subgraph around $v$) and all nodes except $v$. 
% Example: molecule function prediction, interested in the true effect of a molecule's characteristics and not global effect of features, for which an extension is added.
Several options are possible to cancel out a node's influence on the prediction, such as replacing its feature vector by random or expected values. Here, we decide to isolate the node in the graph, which totally removes its effect on the prediction. Similarly, to neutralise the effect of a feature, as GNNs do not handle missing values, we set it to the dataset expected value. Formally, it translates into:
\begin{align}
    \mathbf{X}'&=\mathbf{X} \text{ ~with~ } \mathbf{X}_v'= \mathbf{M}_F \odot \mathbf{X}_v + (\mathbf{1}-\mathbf{M}_F) \odot \boldsymbol{\mu} \label{eq:data} \\
    \mathbf{A}' &= (\mathbf{M}_N^{\text{ext} \top} \cdot \mathbf{A} \cdot  \mathbf{M}_N^{\text{ext}}) \odot I(\mathbf{A}), \label{eq:data2}
\end{align}
where $\boldsymbol{\mu} = (\mathbb{E}[\mathbf{X}_{*1}], \ldots, \mathbb{E}[\mathbf{X}_{*F}])^\top$  and $I(\cdot)$ captures the indirect effect of $k$-hop neighbours of $v$ ($k>1$), which is often underestimated. Indeed, if a $3$-hop neighbour $w$ is considered alone in a coalition, it becomes disconnected from $v$ in the new graph $\mathcal{G'}$. This prevents us from capturing its indirect impact on the prediction since it does not pass information to $v$ anymore. To remedy this problem, 
% in addition to sampling many coalitions, 
we select one shortest path $\mathcal{P}$ connecting $w$ to $v$ via $\textit{Dijkstra's}$ algorithm, and include $\mathcal{P}$ back in the new graph. To keep the influence of the new nodes (in $\mathcal{P}\setminus \{w,v\}$) switched off, we set their feature vector to mean values obtained by Monte Carlo sampling. The pseudocode is in the Supplementary material.
% Maybe remove paragraph on indirect effect

\par To finalize the perturbation dataset, we pass $\mathbf{z'}=(\mathbf{X}',\mathbf{A'})$ to the GNN model $f$ and store each sample $(\mathbf{z}, f(\mathbf{z'}))$ in a dataset $\mathcal{D}$. $\mathcal{D}$ associates with a subset of nodes and features of $v$ their estimated influence on the original prediction.

\subsection{Explanation generator}

In this section, we build a surrogate model $g$ on the dataset $\mathcal{D} = \{(\mathbf{z}, f(\mathbf{z'}))\}$ and provide it as explanation. More rigorously, an explanation $\mathbf{\phi}$ of $f$ is normally drawn from a set of possible explanations, called interpretable domain $\Omega$. It is the solution of the following optimisation process: $\boldsymbol{\phi} = \arg \min_{g \in \Omega} \mathcal{L}_{f}(g)$, where the loss function attributes a score to each explanation. The choice of $\Omega$ has a large impact on the type and quality of the obtained explanation. In this paper, we choose broadly $\Omega$ to be the set of interpretable models, and more precisely the set of Weighted Linear Regression (WLR).

\par In short, we intend our model to learn to calculate the individual effect of each variable towards the original prediction from the joint effect $f(\mathbf{z'})$, using many different coalitions $S$ of nodes and features. This is made possible by the definition of the input dataset $\mathcal{D}$ and is enforced by a cross entropy loss function:
% See difficulty joint effect - no overlaps 
\begin{align}
\begin{split}
    \mathcal{L}_{f, \boldsymbol{\pi}}(g) &= \sum_{\mathbf{z}} \big( g(\mathbf{z})-f(\mathbf{z'}) \big)^2 \boldsymbol{\pi}_{\mathbf{z}}, \\
    \text{where~~~} \quad \boldsymbol{\pi}_{\mathbf{z}} &= \dfrac{F+N-1}{(F+N) \cdot |\mathbf{z}|} \cdot \binom{F+N-1}{|\mathbf{z}|}^{-1}.
\end{split}
\label{eq:objective}
\end{align}
$\boldsymbol{\pi}$ is a kernel weight that attributes a high weight to samples $\mathbf{z}$ with small or large dimension, or in different terms, groups of features and nodes with few or many elements---since it is easier to capture individual effects from the combined effect in these cases.

\par In the end, we provide the learned parameters of $g$ as explanation. Each coefficient corresponds to a node of the graph or a feature of $v$ and represents its estimated influence on the prediction $f_v(\mathbf{X},\mathbf{A})$. In fact, it approximates the extension of the Shapley value to graphs, as shown in next paragraph.

\subsection{Decomposition model} 

We first justify why it is relevant to extend the Shapley value to graphs. Looking back at the original theory, each player contributing to the total gain is allocated a proportion of that gain depending on its fair contribution. Since a GNN model prediction is fully determined by node feature information ($\mathbf{X}$) and graph structural information ($\mathbf{A}$), both edges/nodes and node features are players that should be considered in explanations. In practice, we extend to graphs the four Axioms defining fairness (please see the Supplementary Material), and redefine how is captured the influence of players (features and nodes) towards the prediction as $\text{\textit{val}}(S)= \mathbb{E}_{\mathbf{X}_{v}}[f_v(\mathbf{X},\mathbf{A}_S)|\mathbf{X}_{vS}= \mathbf{x}_{vS}] - \mathbb{E}[f_v(\mathbf{X},\mathbf{A})]$. $\mathbf{A}_S$ is the adjacency matrix where all nodes in $\overline{S}$ (not in $S$) have been isolated.
% While we can simply isolate a node to remove completely its influence on prediction, as players in a coalitional game, we cannot remove features because GNNs do not handle well missing values and thus take the conditional expectation. 
% Note that past approximation methods for Shapley Values are obsolete for non euclidean settings since they are not specifically designed for GNNs and do not consider graph structure and node features in explanations. 

% Before showing that our approach computes the extension of the Shapley values to graphs, let us clarify a few key points. 

\par Assuming model linearity and feature independence, we show that GraphSVX, in fact, captures via $f(\mathbf{z'})$ the marginal contribution of each coalition $S$ towards the prediction: 
\begin{align*}
    \mathbb{E}_{\mathbf{X}_{v}}[f_v(\mathbf{X},\mathbf{A}_S)|\mathbf{X}_{vS}]  &= \mathbb{E}_{\mathbf{X}_{v\overline{S}}|\mathbf{X}_{vS}}[f_v(\mathbf{X},\mathbf{A}_S)]  \\
    &\approx \mathbb{E}_{\mathbf{X}_{v\overline{S}}} [f_v(\mathbf{X},\mathbf{A}_S)] & \text{by independence} \\
    &\approx f_v(\mathbb{E}_{\mathbf{X}_{v\overline{S}}}[\mathbf{X}],  \mathbf{A}_S ) & \text{by linearity} \\
    &= f_v(\mathbf{X}', \mathbf{A'}), 
\end{align*}
where $\mathbf{A'} = \mathbf{A}_S$ and  $\mathbf{X}'_{ij} = 
\begin{cases}
    \mathbb{E}[\mathbf{X}_{*j}] \mbox{ if $i=v$ and $j \in \overline{S}$ }\\
    \mathbf{X}_{ij} \mbox{ otherwise.}
\end{cases} $

\noindent Using the above, we prove that GraphSVX calculates the Shapley values on graph data. This builds on the fact that Shapley values can be expressed as an additive feature attribution model, as shown by \cite{lundberg2017unified} in the case of tabular data. 
%We however greatly differ from it.
% It builds on the work from \cite{lundberg2017unified}, which showed that Shapley Values can be expressed as additive feature attribution model. 
 
In this perspective, we set $\boldsymbol{\pi}_v$ such that $\boldsymbol{\pi}_v(\mathbf{z}) \rightarrow \infty$ when $|\mathbf{z}| \in \{0,F+N\}$ to enforce the \textit{efficiency} axiom: $g(\mathbf{1}) = f_v(\mathbf{X},\mathbf{A})= \mathbb{E}[f_v(\mathbf{X},\mathbf{A})] + \sum_{i=1}^{F+N} \phi_i$. This holds due to the specific definition of \textsc{Gen} and $g$ (i.e., \textsc{Expl}), where $g(\mathbf{1}) = f_v(\mathbf{X},\mathbf{A})$ and the constant $\phi_0$, also called base value, equals $\mathbb{E}_{\mathbf{X}_{v}}[f_v(\mathbf{X},\mathbf{A}_v)] \approx \mathbb{E}[f_v(\mathbf{X},\mathbf{A})]$, so the mean model prediction. $\mathbf{A}_v$ refers to $\mathbf{A}_{\emptyset}$, where $v$ is isolated.

\vspace{.1cm}

\begin{theorem}
With the above specifications and assumptions, the solution to $\min_{g \in \Omega} \mathcal{L}_{f, \boldsymbol{\pi}}(g)$ under Eq. \eqref{eq:objective} is a unique explanation model $g$ whose parameters compute the extension of the Shapley values to graphs.
\label{theorem}
\end{theorem}

\begin{proof}
Please see Appendix \ref{A_proof}.  
\end{proof}

\subsection{Efficient approximation specific to GNNs}
\label{efficient_approx}

Similarly to the euclidean case, the exact computation of the Shapley values becomes intractable due to the number of possible coalitions required. Especially that we consider jointly features and nodes, which augments exponentially the complexity of the problem. To remedy this, we derive an efficient approximation via a smart mask generator.

\par Firstly, we reduce the number of nodes and features initially considered to $D\leq N$ and $B \leq F$ respectively, without impacting performance. Indeed, for a GNN model with $k$ layers, only $k$-hop neighbours of $v$ can influence the prediction for $v$, and thus receive a non-zero Shapley value. All others are allocated a null importance according to the \textit{dummy} axiom\footnote{Axiom: If $\forall S \in \mathcal{P}(\{1,\ldots,p\}) \text{ and } j \notin S$, $\text{\textit{val}}(S\cup \{j\}) = \text{\textit{val}}(S)$, $\text{ then } \phi_j(\text{\textit{val}}) = 0$.} and can therefore be discarded. Similarly, any feature $j$ of $v$ whose value is comprised in the confidence interval $I_j  = [\mu_j - \lambda \cdot \sigma_j, \mu_j +  \lambda \cdot \sigma_j]$ around the mean value $\mu_j$ can be discarded, where $\sigma_j$ is the corresponding standard deviation and $\lambda$ a constant.

\par The complexity is now $\mathcal{O}(2^{B+D})$ and we further drive it down to $\mathcal{O}(2^B + 2^D)$ by sampling separately masks of nodes and features, while still considering them jointly in $g$. In other words, instead of studying the influence of possible combinations of nodes and features, we consider all combinations of features with no nodes selected, and all combinations of nodes with all features included: $(2^B + 2^D)$. 
We observe empirically that it achieves identical explanations with fewer samples, while it seems to be more intuitive to capture the effect of nodes and features on prediction (expressed by Axiom \ref{axiom:rel-eff}).

\begin{axiom}[Relative efficiency] Node contribution to predictions can be separated from feature contribution, and their sum decomposes the prediction with respect to the average one:
    $\begin{cases}
    \sum_{j=1}^{B} \phi_j = f_v(\mathbf{X}, \mathbf{A}_v) - \mathbb{E}[f_v(\mathbf{X},\mathbf{A})] \\
    \sum_{i=1}^{D} \phi_{B+i} = f_v(\mathbf{X},\mathbf{A}) - f_v(\mathbf{X},\mathbf{A}_v). \\
    \end{cases}$
\label{axiom:rel-eff}    
\end{axiom}

Lastly, we approximate explanations using $P \ll 2^B + 2^D$ samples, where $P$ is sufficient to obtain a good approximation. We reduce $P$ by greatly improving \textsc{Mask}, as evoked in Sec. \ref{MGG}. Assuming we have a budget of $P$ samples, we develop a smart space allocation algorithm to draw in priority coalitions of order $k$, where $k$ starts at $0$ and is incremented when all coalitions of the order are sampled. This means that we sample in priority coalitions with high weight, so with nearly all or very few players. If they cannot all be chosen (for current $k$) due to space constraints, we proceed to a smart sampling that favours unseen players. The pseudocode and an evaluation of its efficiency lie in Appendix \ref{A_GSVX} and \ref{A_eval} respectively.

\subsection{Desirable properties of explanations}

In the end, GraphSVX generates fairly distributed explanations $\sum_j \phi_j = f_v(\mathbf{X},\mathbf{A})$, where each $\phi_j$ approximates the average marginal contribution of a node or feature $j$ towards the explained GNN prediction (with respect to the average prediction $\phi_0$). By definition, the resulting explanation is unique, consistent, and stable. It is also truthful and robust to noise, as shown in Sec. \ref{_Eval}. The last focus of this paper is to make them more selective, global, contrastive and social; as we aim to design an explainer with desirable properties. A few aspects are detailed here. \\

\noindent
\textbf{Contrastive}. Explanations are contrastive already as they yield the contribution of a variable with respect to the average prediction $\phi_0 = \mathbb{E}[f(\mathbf{X},\mathbf{A})]$.
To go futher and explain an instance with respect to another one, we could substitute $\mathbf{X}_v$ in Eq. \eqref{eq:data} by $\mathbf{X}_v'= \mathbf{M}_F \odot \mathbf{X}_v + (\mathbf{1}-\mathbf{M}_F) \odot \boldsymbol{\xi}$, with $\boldsymbol{\xi}$ being the feature vector of a specific node $w$, or of a fictive representative instance from class $C$. \\

\noindent
\textbf{Global}. We derive explanations for a subset $U$ of nodes instead of a single node $v$, following the same pipeline. The neighbourhood changes to $\bigcup_i^U \mathcal{N}_i$, Eq. \eqref{eq:data} now updates $\mathbf{X}_U$ instead of $\mathbf{X}_v$ and $f(\mathbf{z'})$ is calculated as the average prediction score for nodes in $U$. Also, towards a more global understanding, we can output the global importance of each feature $j$ on $v$'s prediction by enforcing in Eq. \eqref{eq:data} $\mathbf{X}_{\mathcal{N}_v \cup \{v\},j}$ to a mean value obtained by Monte Carlo sampling on the dataset, when $\mathbf{z}_j=0$. This holds when we discard node importance, otherwise the overlapping effects between nodes and features render the process obsolete. \\

% \textbf{Confirmation bias}. Finally, explanations are consistent with prior beliefs of the explainee \cite{nickerson1998confirmation}. This last bit is very difficult to embed in an explainer. We more or less implicitly enforce it via the monotonicity constraint imposed by the weighted linear regression, which keeps the relation between input features and output monotone.

% \textbf{Selective}. Instead of  fitting a WLR explanation model $g$, we fit instead a weighted Lasso regression, which produces sparser results, more easily comprehensible by humans \cite{ustun2014methods}. 

\noindent
\textbf{Graph classification}. Until now, we had focused on node classification but the exact same principle applies for graph classification. We simply look at $f(\mathbf{X},\mathbf{A}) \in \mathbb{R}$ instead of $f_v(\mathbf{X},\mathbf{A})$, derive explanations for all nodes or all features (not both) by considering features across the whole dataset instead of features of $v$, like our global extension. \\

\label{_Method}

\section{Experimental Evaluation}
In this section, we conduct several experiments designed to determine the quality of our explanation method, using synthetic and real world datasets, on both node and graph classification tasks. We first study the effectiveness of GraphSVX in presence of ground truth explanations.
We then show how our explainer generalises to more complex real world datasets with no ground truth, by testing GraphSVX's ability to filter noisy features and noisy nodes from explanations.
Detailed dataset statistics, hyper-parameter tuning, properties' check and further experimental results including ablation study, are given in Appendix \ref{A_eval}.

\subsection{Synthetic and real datasets with ground truth}

\textbf{Synthetic node classification task}. We follow the same setting as \cite{luo2020parameterized} and \cite{ying2019gnnexplainer}, where four kinds of datasets are constructed. Each input graph is a combination of a base graph together with a set of motifs, which both differ across datasets. The label of each node is determined based on its belonging and role in the motif. As a consequence, the explanation for a node in a motif should be the nodes in the same motif, which creates ground truth explanation. This ground truth can be used to measure the performance of an explainer via an accuracy metric. \\

\noindent
\textbf{Synthetic and real-world graph classification task.}
With a similar evaluation perspective, we measure the effectiveness of our explainer on graph classification, also using ground truth. We use a synthetic dataset \textsl{BA-2motifs} that resembles the previous ones, and a real life dataset called \textsl{MUTAG}. It consists  of $4,337$ molecule graphs, each assigned to one of 2 classes based on its mutagenic  effect~\cite{riesen2008iam}. As discussed in~\cite{debnath1991structure}, carbon rings with  groups $NH_2$ or $NO_2$ are known to be mutagenic, and could therefore be used as ground truth. \\

\noindent
\textbf{Baselines}. We compare the performance of GraphSVX to the main explanation baselines that incorporate graph structure in explanations, namely GNNExplainer, PGExplainer and PGM-Explainer. GraphLIME and XGNN are not applicable here, since they do not provide graph structure explanations for such tasks. \\

\noindent
\textbf{Experimental setup and metrics}. We train the same GNN model -- 3 graph convolution blocks with $20$ hidden units, (maxpooling) and a fully connected classification layer -- on every dataset during $1,000$ epochs, with relu activation, Adam optimizer and initial learning rate $0.001$. The performance is measured with an accuracy metric (node or edge accuracy depending on the nature of explanations) on top-$k$ explanations, where $k$ is equal to the ground truth dimension. More precisely, we formalise the evaluation as a binary classification of nodes (or edges) where nodes (or edges) inside motifs are positive, and the rest negative. \\
% Depending on the evaluation coefficient, we select top-k explanations (highest absolute value of importance coefficient)

\begin{table*}[t]
% \vspace{-0.4cm}
\begin{scriptsize}
     \begin{center}
     \begin{adjustbox}{max width=\textwidth}
     \begin{tabular}{ >{\centering\arraybackslash}m{2.3cm}  >{\centering\arraybackslash}m{1.5cm}  >{\centering\arraybackslash}m{1.9cm}  >{\centering\arraybackslash}m{1.7cm} >{\centering\arraybackslash}m{1.5cm}  >{\centering\arraybackslash}m{1.5cm} >{\centering\arraybackslash}m{1.5cm} }
    %\toprule
    & \multicolumn{4}{c}{\scriptsize{\textbf{Node Classification}}} & \multicolumn{2}{c}{\scriptsize{\textbf{Graph Classification}}}\\
    \cmidrule(l){2-5} \cmidrule(l){6-7} 
      & \textsl{BA-Shapes} & \scriptsize{\textsl{BA-Community}} & \textsl{Tree-Cycles} & \textsl{Tree-Grid} & \textsl{BA-2motifs}    & \vspace{.2cm}\textsl{MUTAG}   \vspace{0.5em} \\ 
    Base &
     \includegraphics[width=0.3in, height=0.3in]{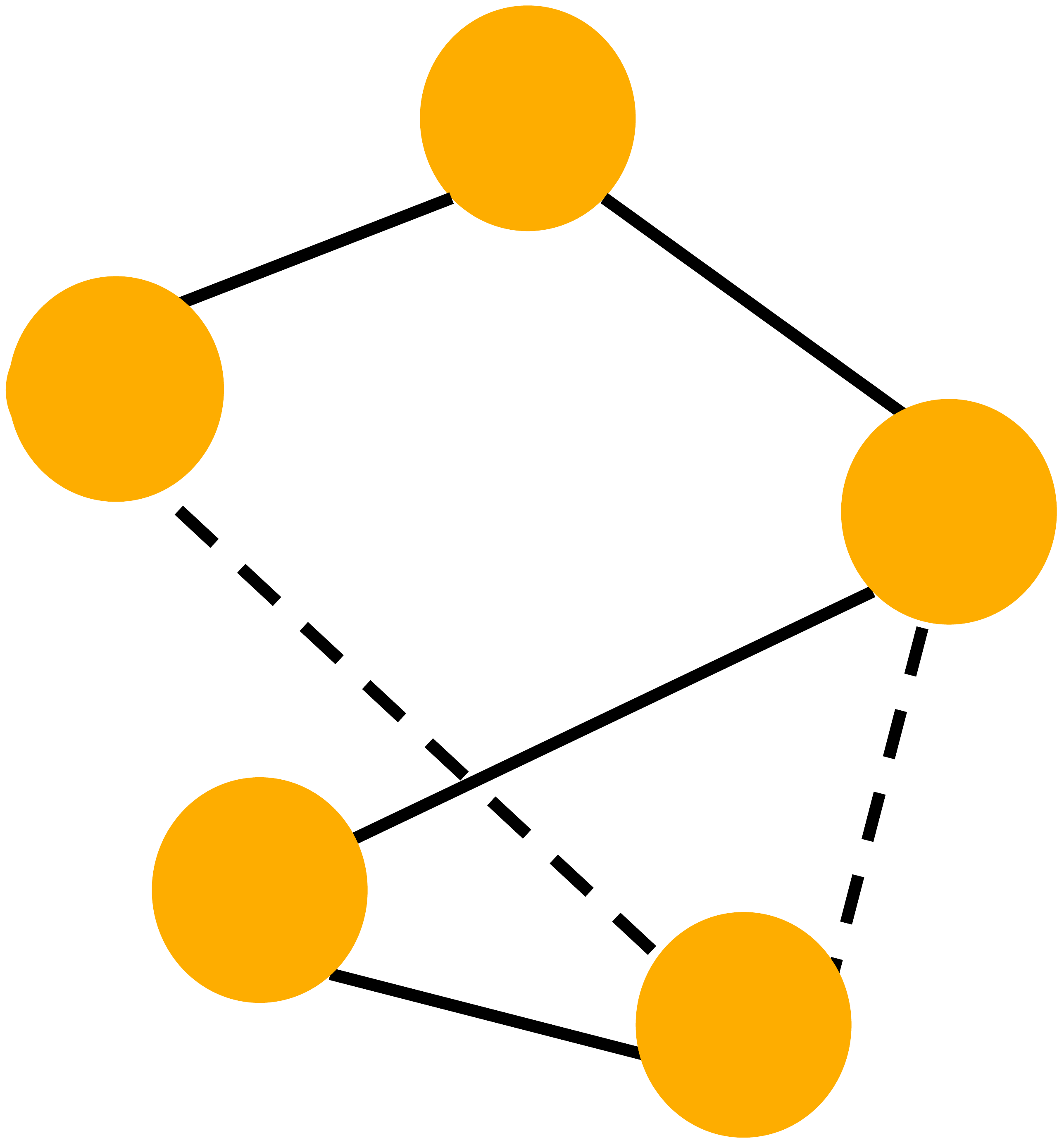}&
     \multirow{2}{*}[2.8ex]{\includegraphics[width=0.9in, height=0.7in]{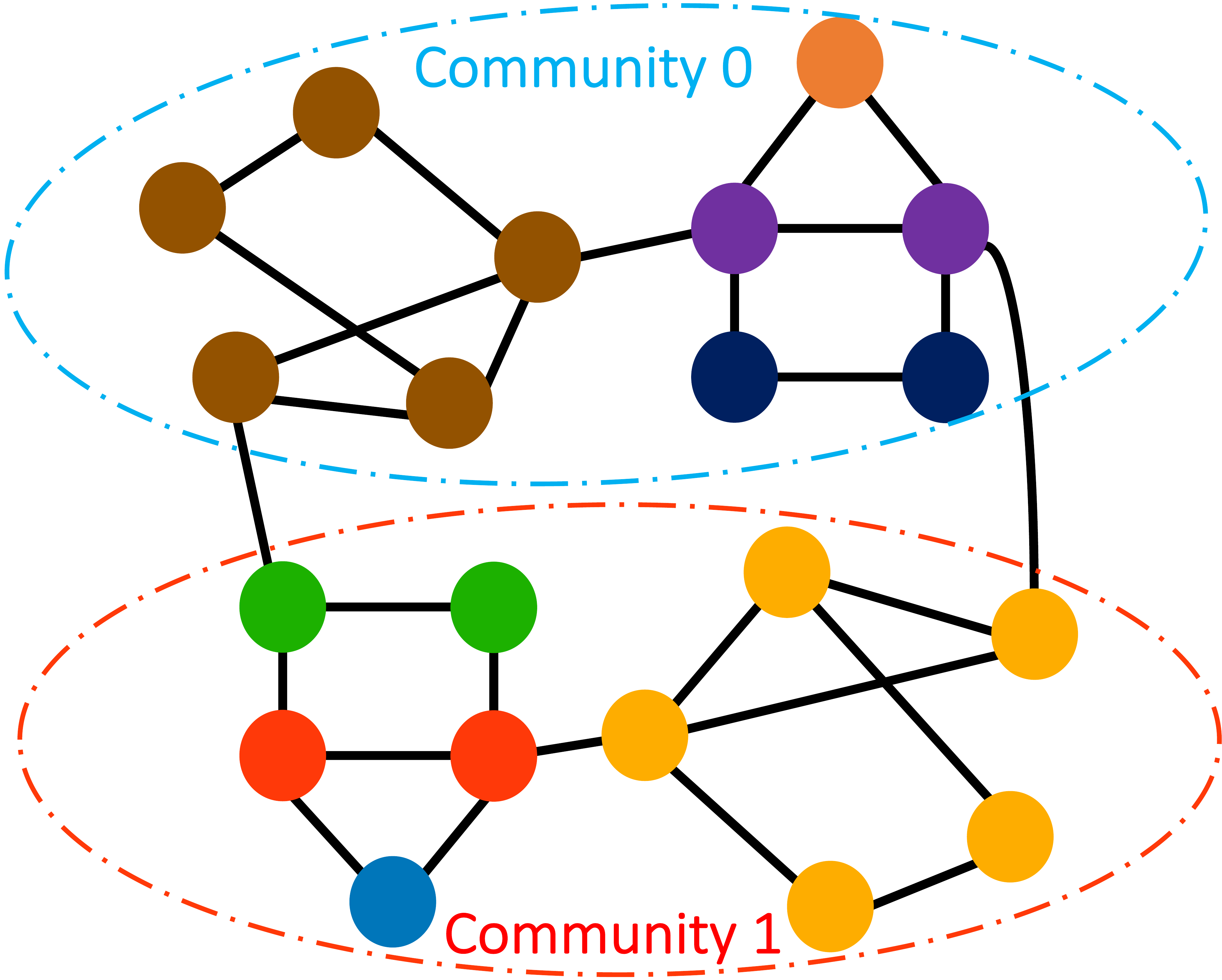}}&
     \includegraphics[width=0.35in, height=0.3in]{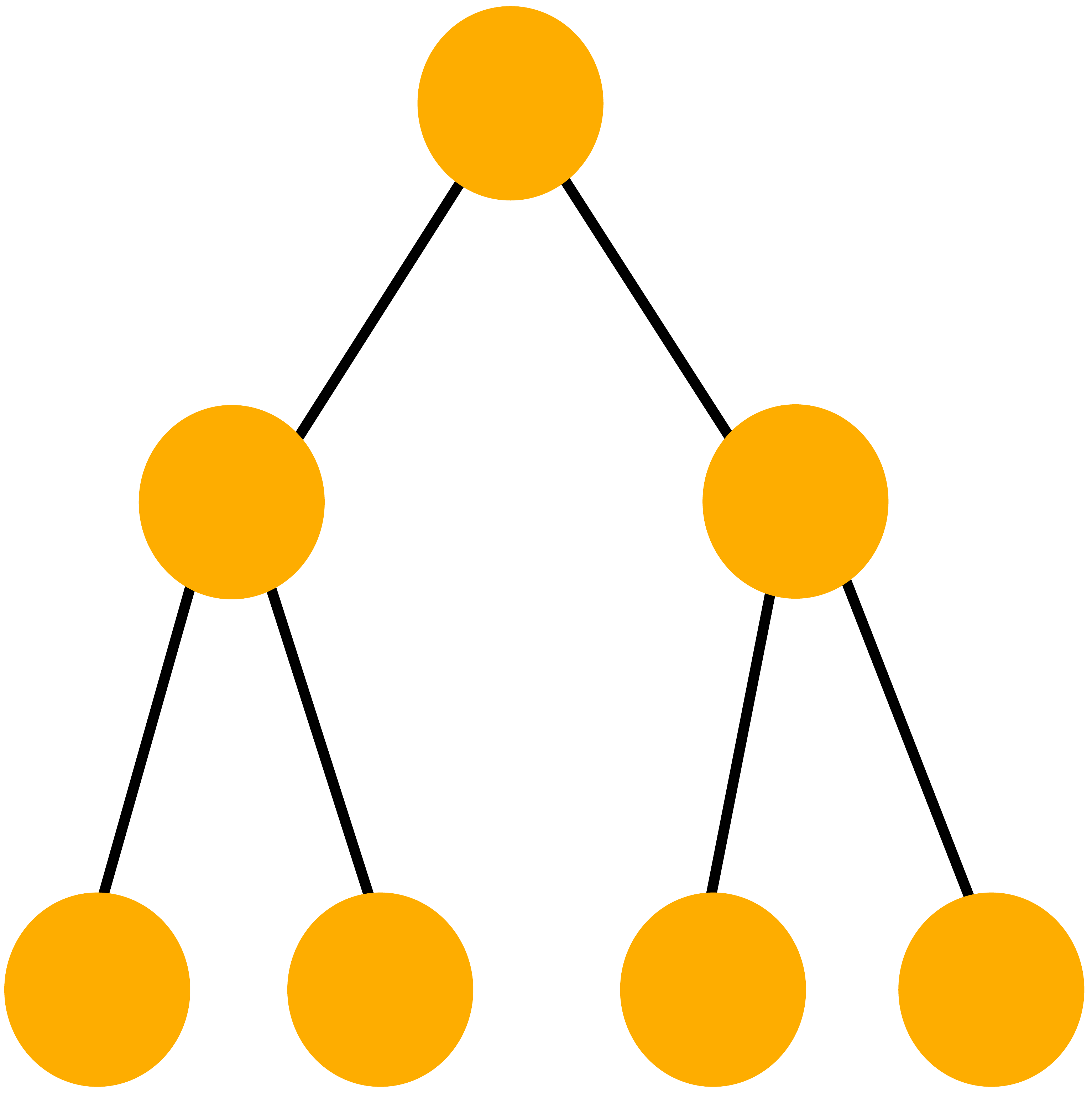}&
     \includegraphics[width=0.35in, height=0.3in]{figures/motifs/Tree.pdf}&
    \includegraphics[width=0.27in, height=0.3in]{figures/motifs/BA.pdf} &
    \includegraphics[width=0.28in, height=0.3in]{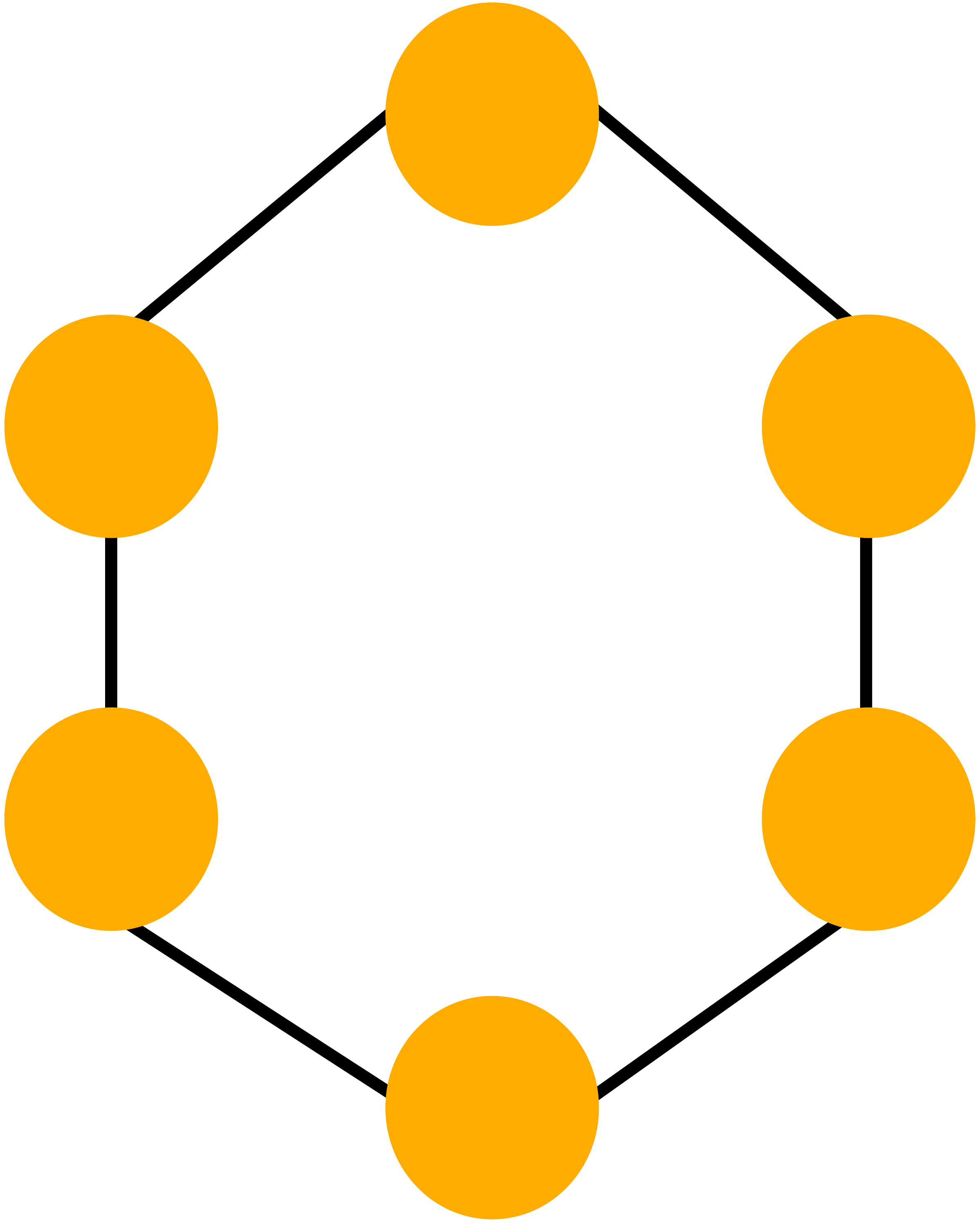} \vspace{0.5em} \\
     
    Motifs &
     \includegraphics[width=0.3in, height=0.3in]{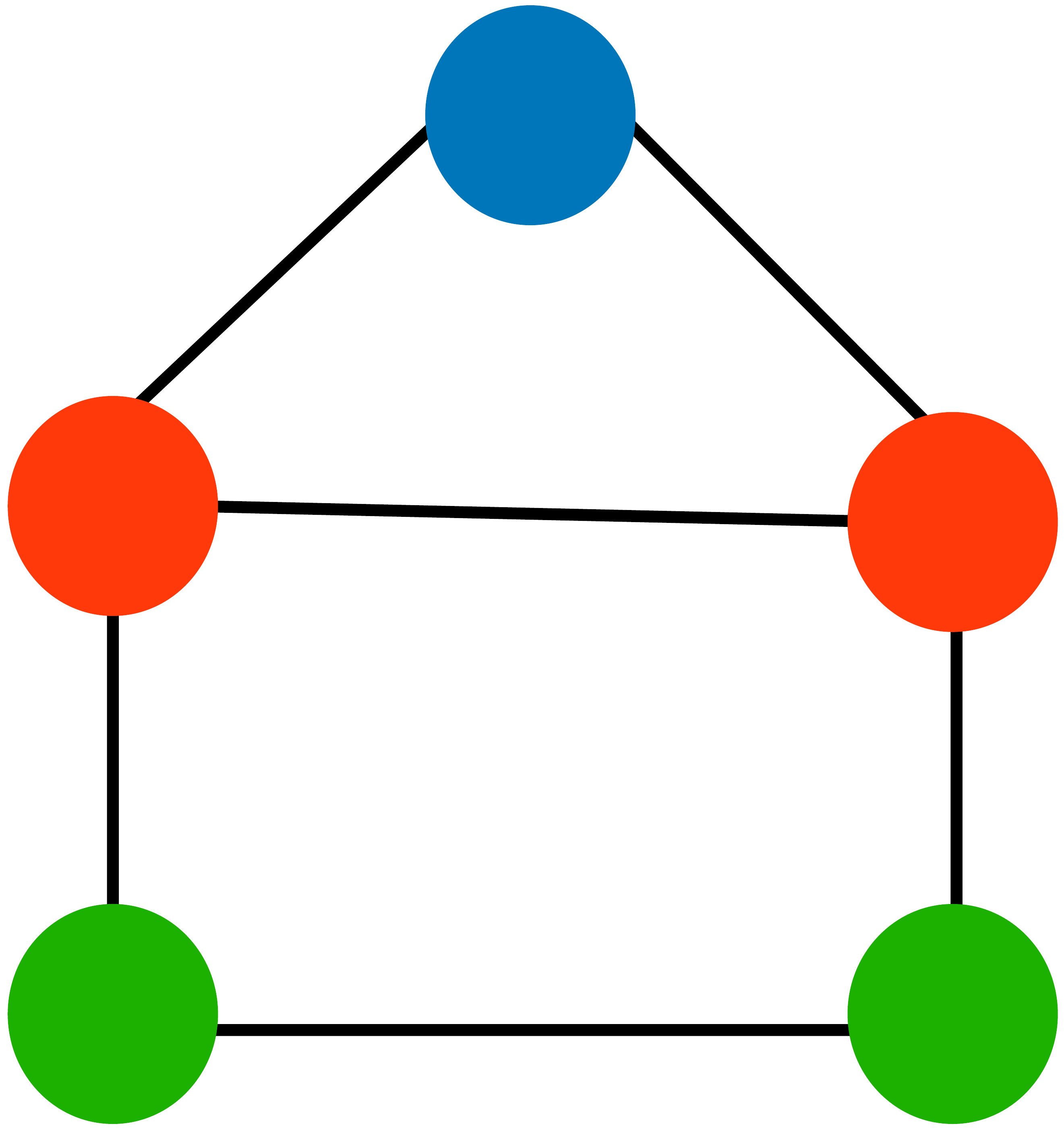}&&
     \includegraphics[width=0.27in, height=0.3in]{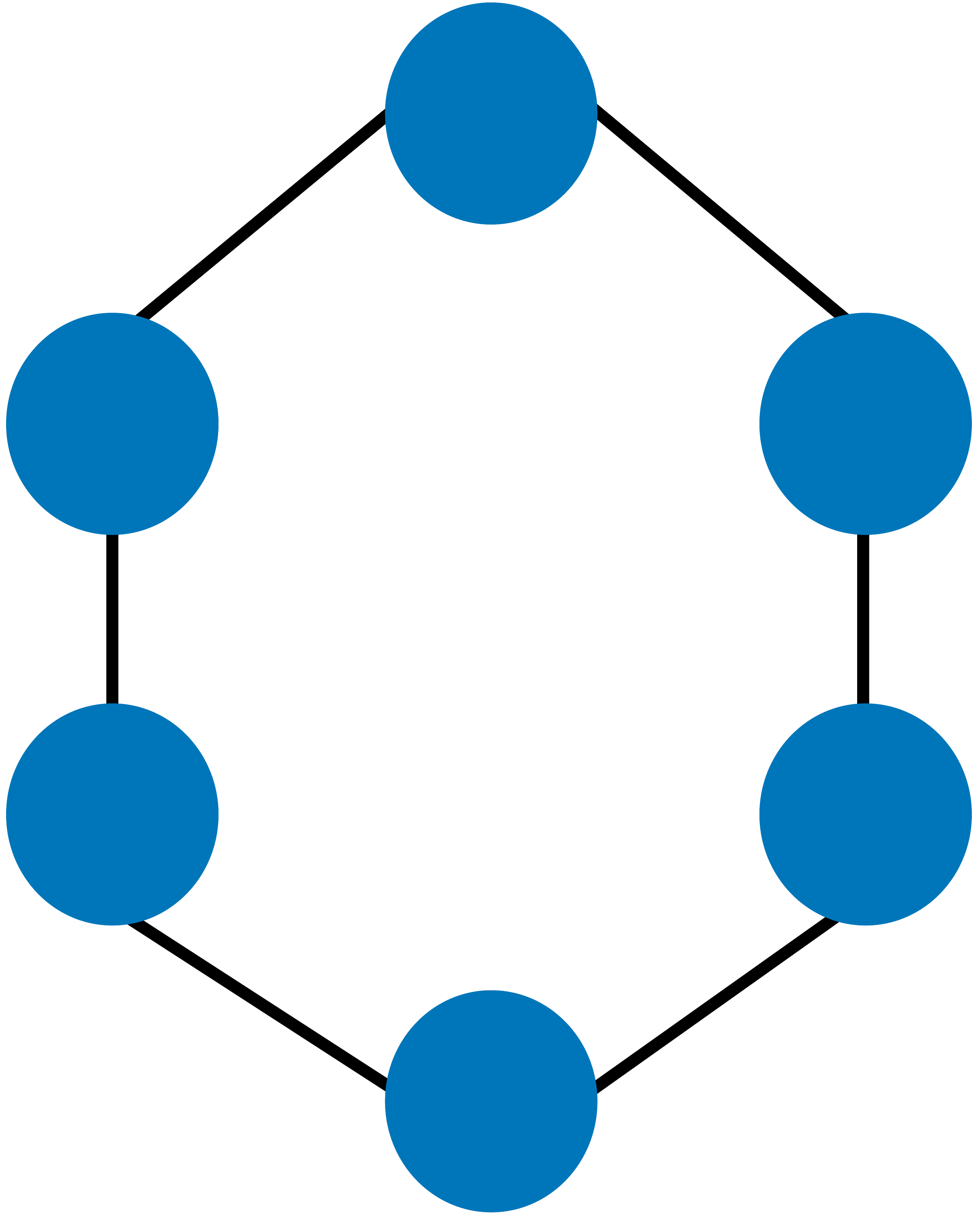}&
     \includegraphics[width=0.28in, height=0.28in]{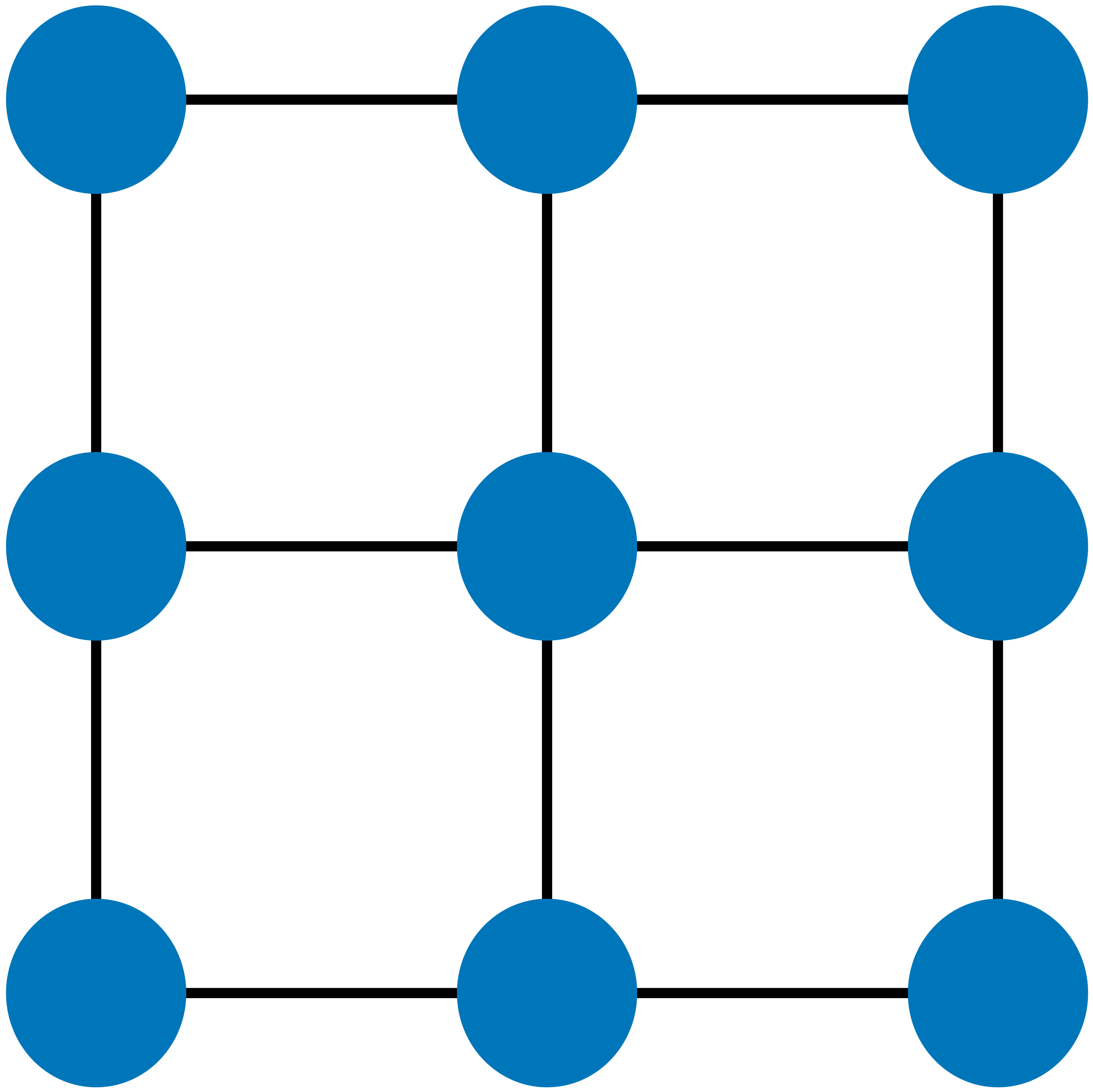}&
    \includegraphics[width=0.5in, height=0.32in]{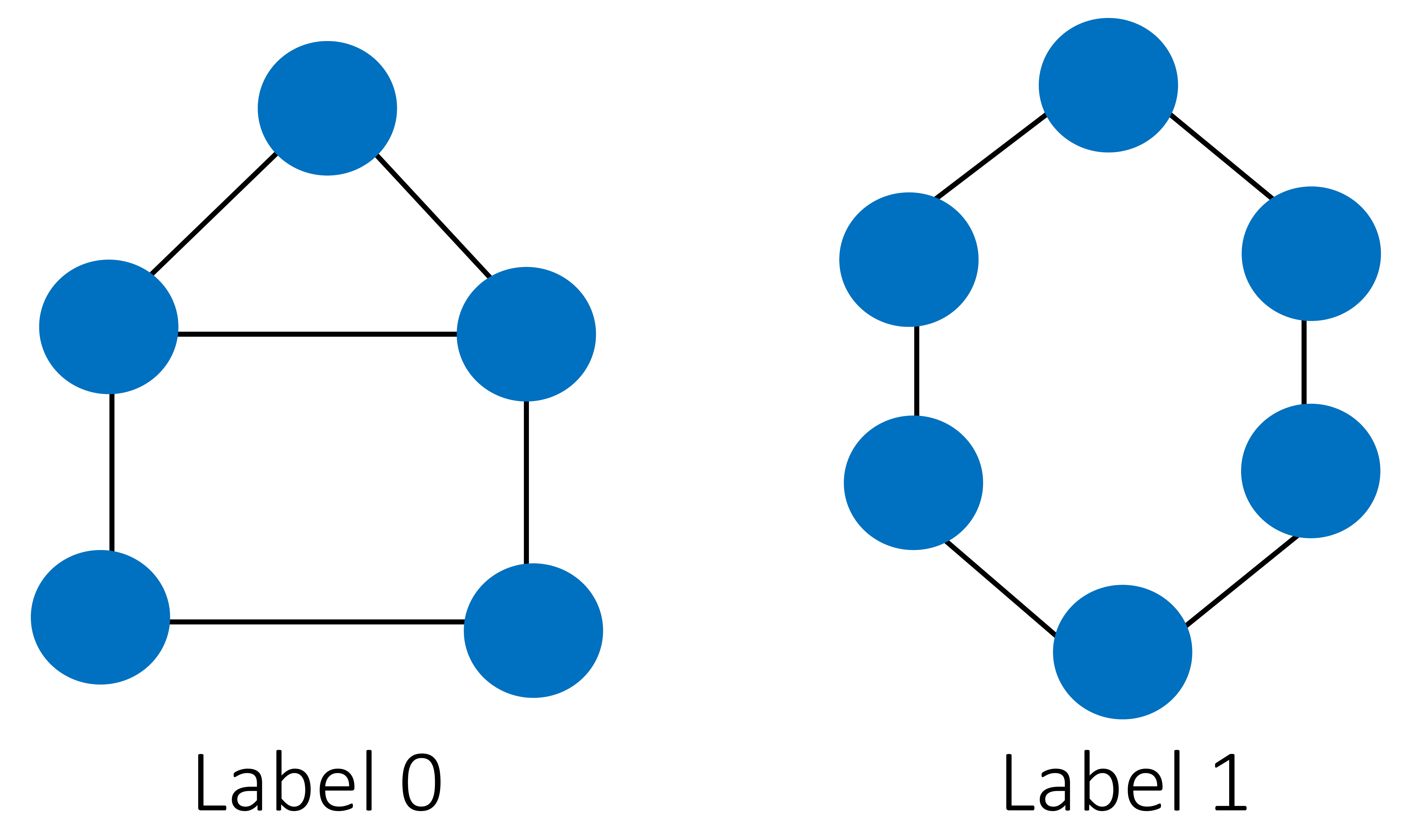} &
     \includegraphics[width=0.4in, height=0.28in]{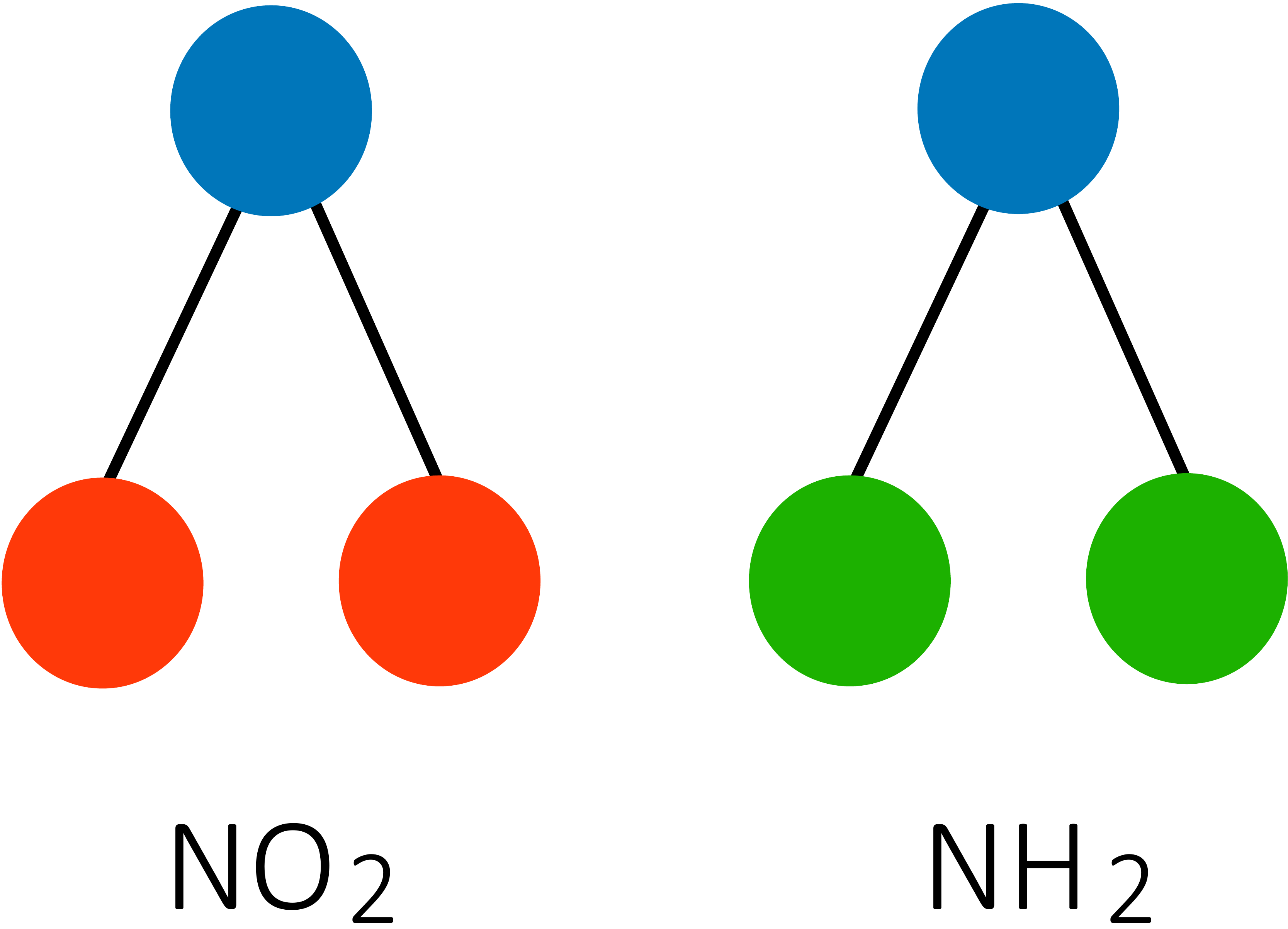} \vspace{0.6em}\\
     
    Features & None & $\mathcal{N}(\mu_l,\sigma_l)$ & None & None & None & Atom types \\[0.2cm]
        %\hline
    \multicolumn{7}{c}{\normalfont{\textbf{Visualization}}}\\
        %  \hline
        \smallskip
        %  \hline
    Explanations by GraphSVX &
     \includegraphics[width=1.6cm, height=1.6cm]{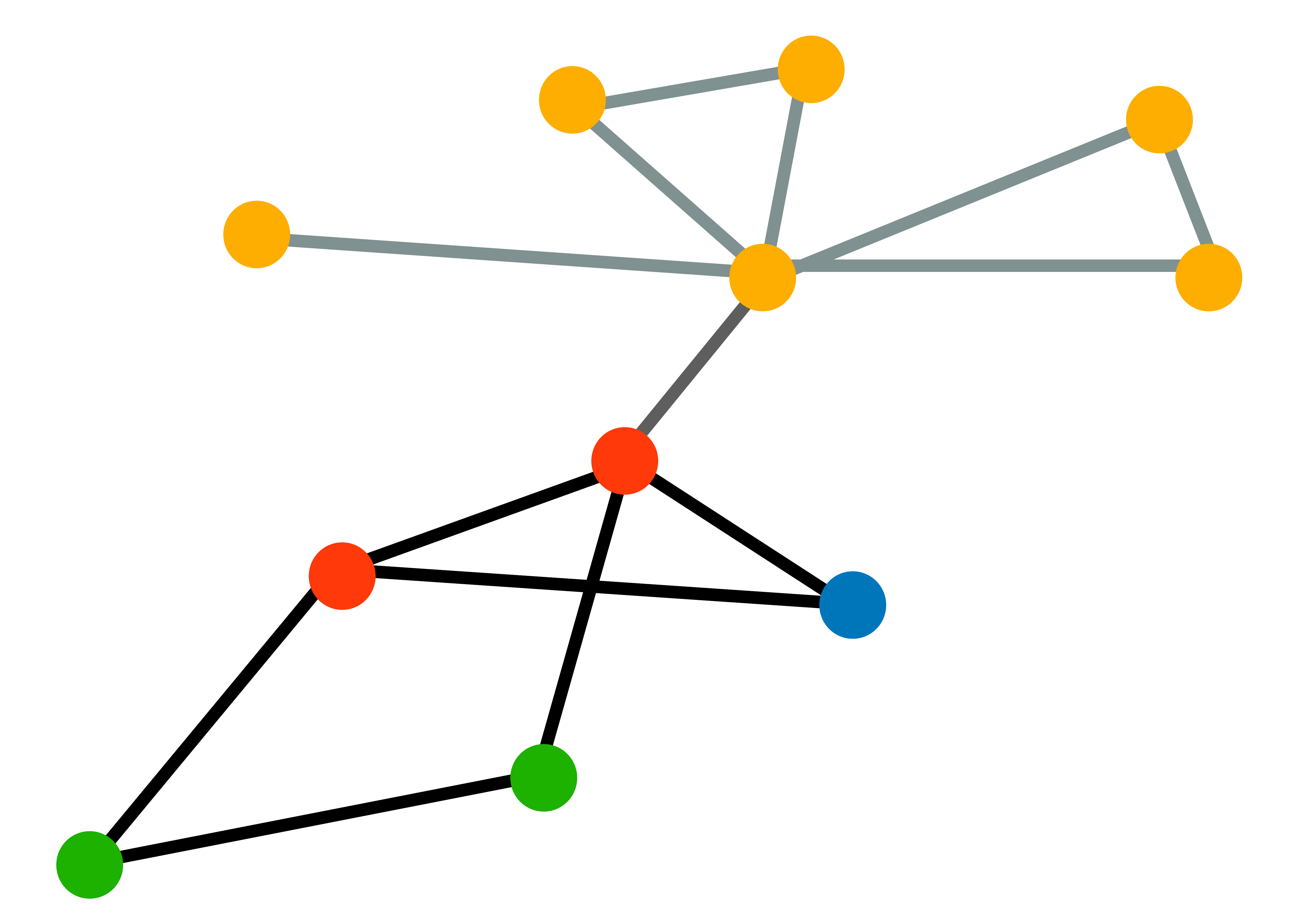}&
     \includegraphics[width=1.6cm, height=1.6cm]{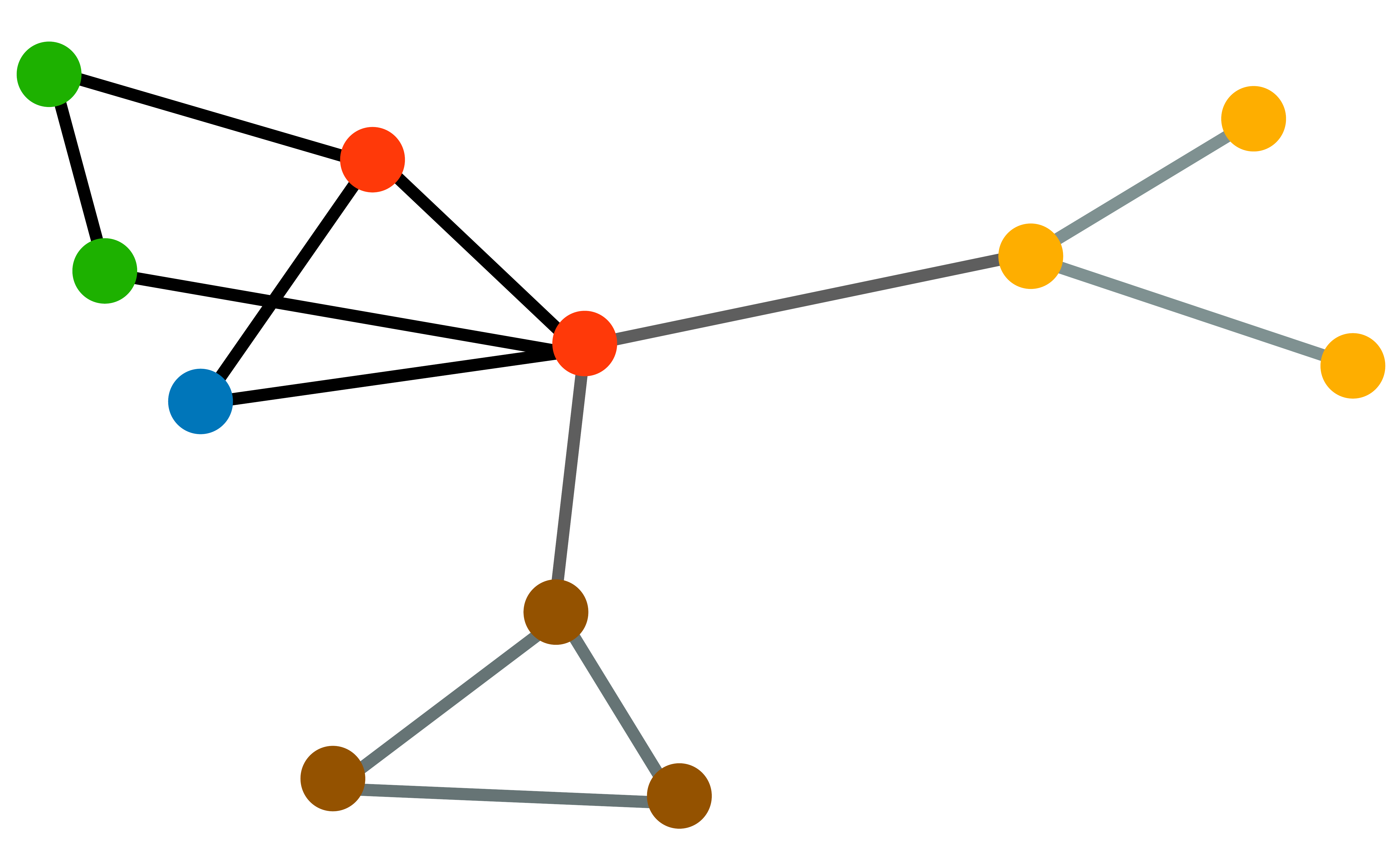}&
     \includegraphics[width=1.6cm, height=1.6cm]{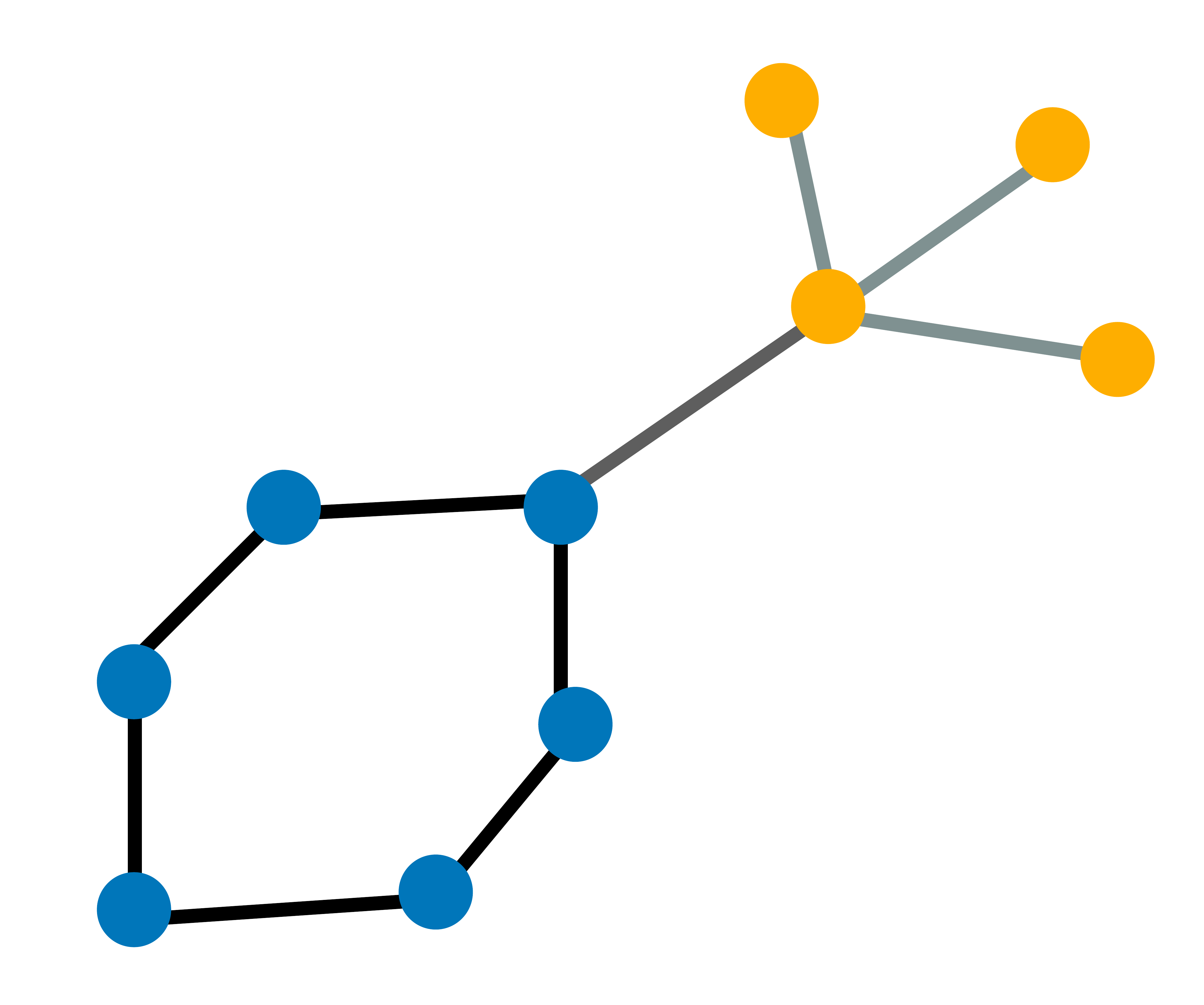}&
     \includegraphics[width=1.6cm, height=1.6cm]{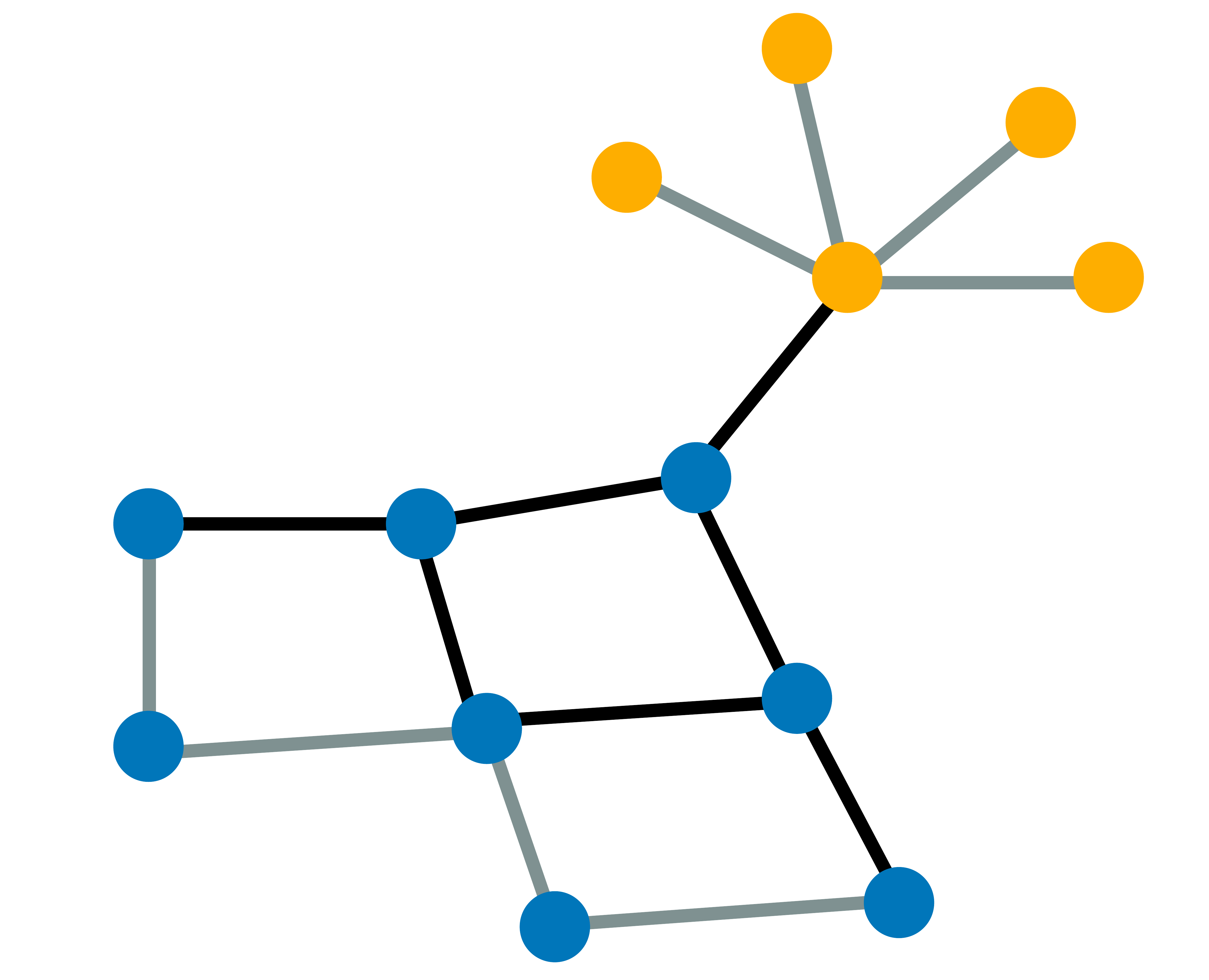}&
    \includegraphics[width=1.6cm, height=1.6cm]{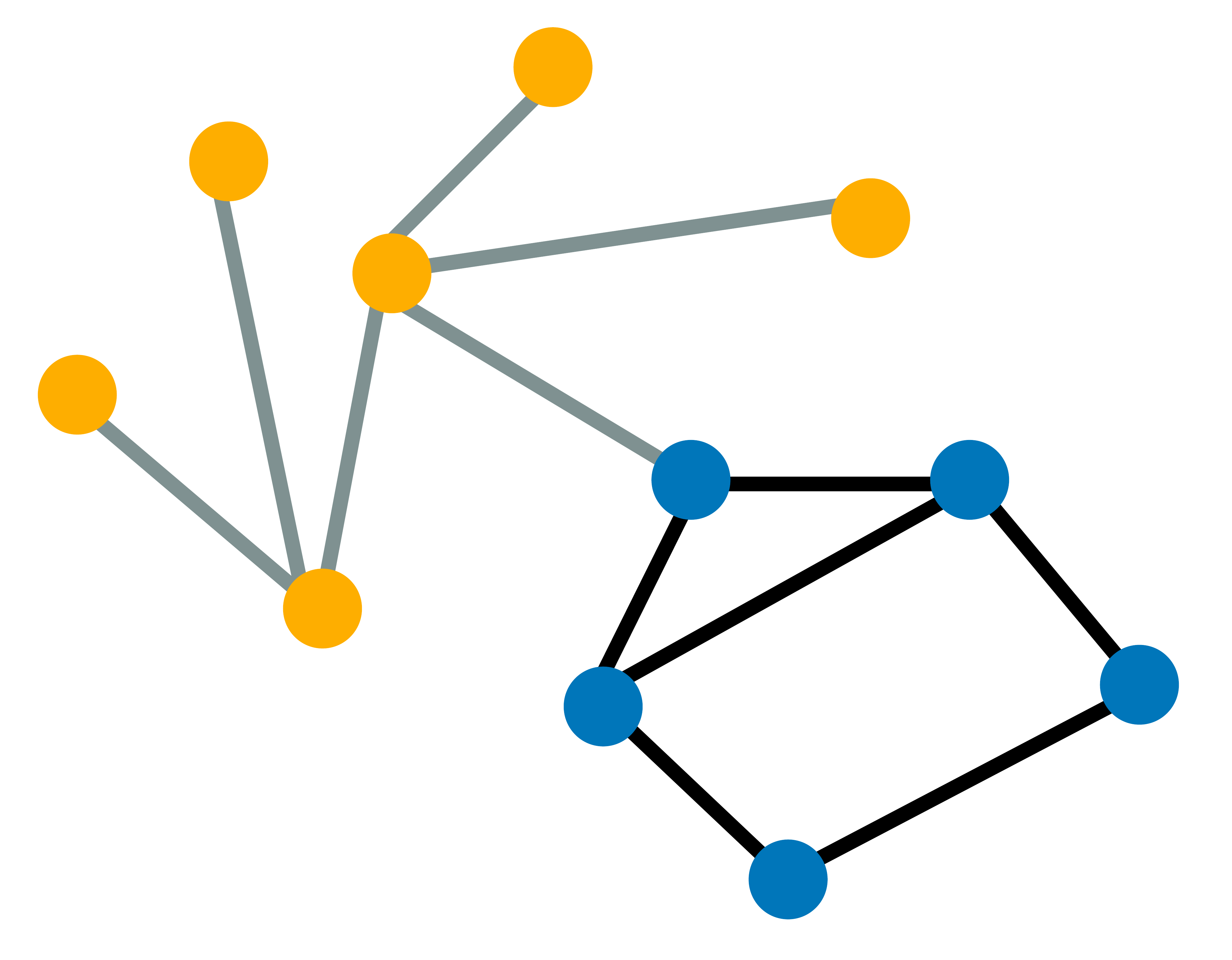} &
    \includegraphics[width=1.6cm, height=1.6cm]{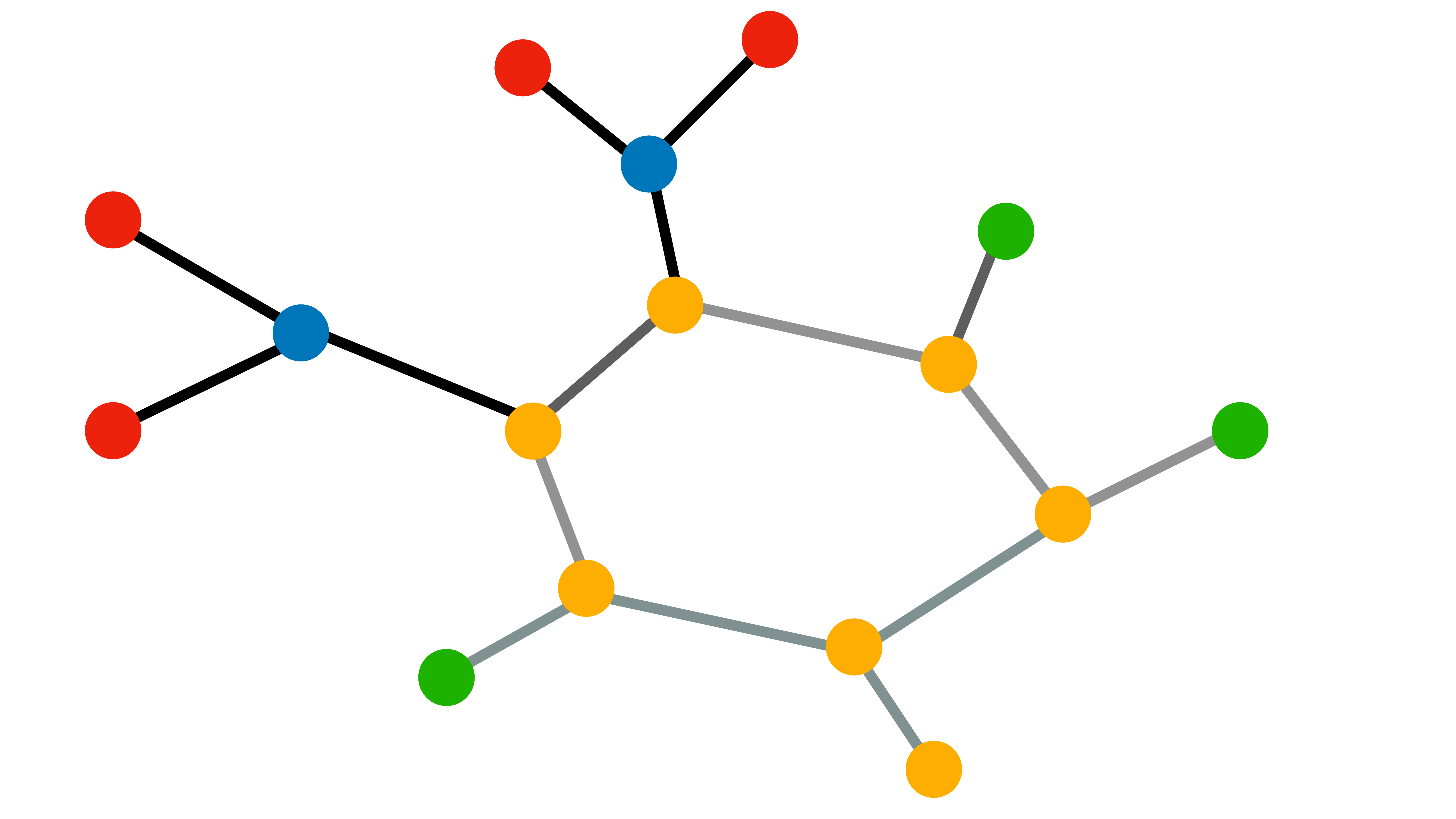} \\   
    \\ 
    \multicolumn{7}{c}{\small{\textbf{Explanation Accuracy}}}\\
    \midrule
     {\footnotesize{GNNExplainer}} & {\footnotesize{0.83}} & {\footnotesize{0.75}}   & {\footnotesize{0.86}}  & {\footnotesize{0.84}} & {\footnotesize{0.68}} & {\footnotesize{\smash{\raisebox{-2pt}{0.65}}}} \vspace{0.3em} \\ 
   % \hline
    \rowcolor{cycle2!8} {\footnotesize{PGM-Explainer}} & {\footnotesize{0.96}}  & {\footnotesize{0.92}}  & {\footnotesize{0.95}}  & {\footnotesize{0.87}} & {\footnotesize{0.91}} & {\footnotesize{\smash{\raisebox{-3pt}{0.72}}}} \vspace{0.3em} \\ 
    % \hline
    {\footnotesize{PGExplainer}} & {\footnotesize{0.92}} & {\footnotesize{0.81}}   & {\footnotesize{0.96}}  & {\footnotesize{0.88}}  & {\footnotesize{0.85}} & {\footnotesize{\smash{\raisebox{-3pt}{\textbf{0.79}}}}}  \vspace{0.3em}\\
    % \hline
    \rowcolor{cycle2!8} {\footnotesize{GraphSVX}} & {\footnotesize{\textbf{0.99}}} & {\footnotesize{\textbf{0.93}}}   & {\footnotesize{\textbf{0.97}}} & {\footnotesize{\textbf{0.93}}} & {\footnotesize{\textbf{0.99}}} & {\footnotesize{\smash{\raisebox{-3pt}{0.77}}}} \vspace{0.3em} \\ 
     \midrule
      \end{tabular}
      \end{adjustbox}
      \caption{Evaluation of GraphSVX and baseline GNN explainers on various datasets. The top part describes the construction of each dataset, with its base graph, the motif added, and the node features generated. Node labels are represented by colors. Then, we provide a visualisation of GraphSVX's explanations, where an important substructure is drawn in bold, as well as a quantitative evaluation based on the accuracy metric.}
        \vspace{-2.5em}
      \label{tab:results}
      \end{center}
\end{scriptsize}
\end{table*}
\normalfont

\noindent
\textbf{Results}. The results on both synthetic and real-life datasets are summarized in Table \ref{tab:results}. As shown both visually and quantitatively, GraphSVX correctly identifies essential graph structure, outperforming the leading baselines on all but one task, in addition to offering higher theoretical guarantees and human-friendly explanations. On \textsl{MUTAG}, the special nature of the dataset and ground truth favours edge explanation methods, which capture slightly more information than node explainers. Hence, we expect PGExplainer to perform better. For \textsl{BA-Community}, GraphSVX demonstrates its ability to identify relevant features and nodes together, as it also identifies important node features with 100\% accuracy. In terms of efficiency, our explainer is slower than the scalable PGExplainer despite our efficient approximation, but is often comparable to GNNExplainer. Running time experiments as well as existence of desirable properties (certainty, stability, consistency, comprehensibility, etc.) are given in Appendix \ref{A_eval} and \ref{A_prop}. 
% Lastly, it comes with qualitative visualisations as well as nice properties, whose existence is checked in the Supplementary Material. 

\subsection{Real-world datasets without ground truth}

Previous experiments involve mostly synthetic datasets, which are not totally representative of real-life scenarios. Hence, in this section, we evaluate GraphSVX on two real-world datasets without ground truth explanations: \textsl{Cora} and \textsl{PubMed}. Instead of looking if the explainer provides the correct explanation, we check that it does not provide a bad one. In particular, we introduce noisy features and nodes to the dataset, train a new GNN on the latter (which we verify do not leverage these noisy variables) and observe if our explainer includes them in explanations. In different terms, we investigate if the explainer filters useless features/nodes in complex datasets, selecting only relevant information in explanations. \\

\noindent
\textbf{Datasets}. \textsl{Cora} is a citation graph where nodes represent articles and edges represent citations between pairs of papers. The task involved is document classification where the goal is to categorise each paper into one out of seven categories. Each feature indicates the absence/presence of the corresponding term in its abstract. \textsl{PubMed} is also a publication dataset with three classes and 500 features, each indicating the TF-IDF value of the corresponding word. \\

\noindent
\textbf{Noisy features}. Concretely, we artificially add 20\% of new “noisy” features to the dataset. We define these new features using existing ones' distribution. We re-train a 2-layer GCN and a 2-layer GAT model on this noisy data, whose test accuracy is above 75\%. The detailed experimental settings are provided in Appendix \ref{A_eval}. We then produce explanations for 50 test samples using different explainer baselines, on \textsl{Cora} and \textsl{PubMed}, and we compare their performance by assessing how many noisy features are included in explanations among top-$k$ features. Ultimately, we compare the resulting frequency distributions using a kernel density estimator (KDE). Intuitively, since features are noisy, they are not used by the GNN model, and thus are unimportant. Therefore, the less noisy features are included in the explanation, the better the explainer. 

Baselines include GNNExplainer, GraphLIME (described previously) as well as the well-known SHAP \cite{lundberg2017unified} and LIME \cite{ribeiro2016should} models. We also compare GraphSVX to a method based on a Greedy procedure, which greedily removes the most contributory features/nodes of the prediction until the prediction changes, and to the Random procedure, which randomly selects $k$ features/nodes as the explanations for the prediction being explained. 

The results are depicted in Fig. \ref{KDE} (a)-(b). For all GNNs and on all datasets, the  number of noisy features selected by GraphSVX is close to zero, and in general lower than existing baselines---demonstrating its robustness to noise. \\

\begin{figure}[t]
     \centering

\subfigure[\textsl{Cora}-features]{
\begin{minipage}[t]{0.235\linewidth}
\centering
\includegraphics[width=\linewidth]{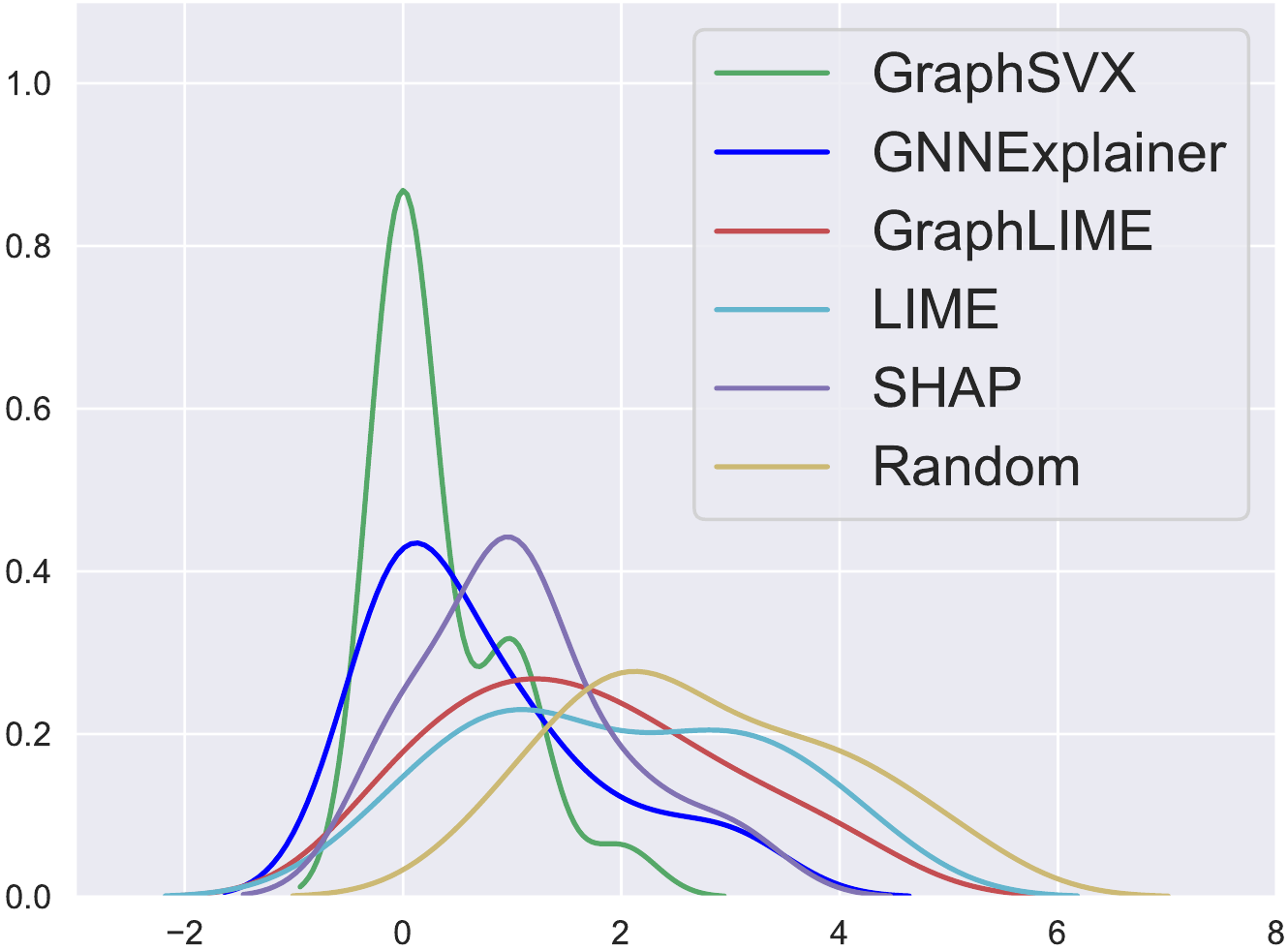}
\end{minipage}%
}%
%\hspace{-.8cm}
\vspace{-.1cm}
\subfigure[\textsl{PubMed}-features]{
\begin{minipage}[t]{0.235\linewidth}
\centering
\includegraphics[width=\linewidth]{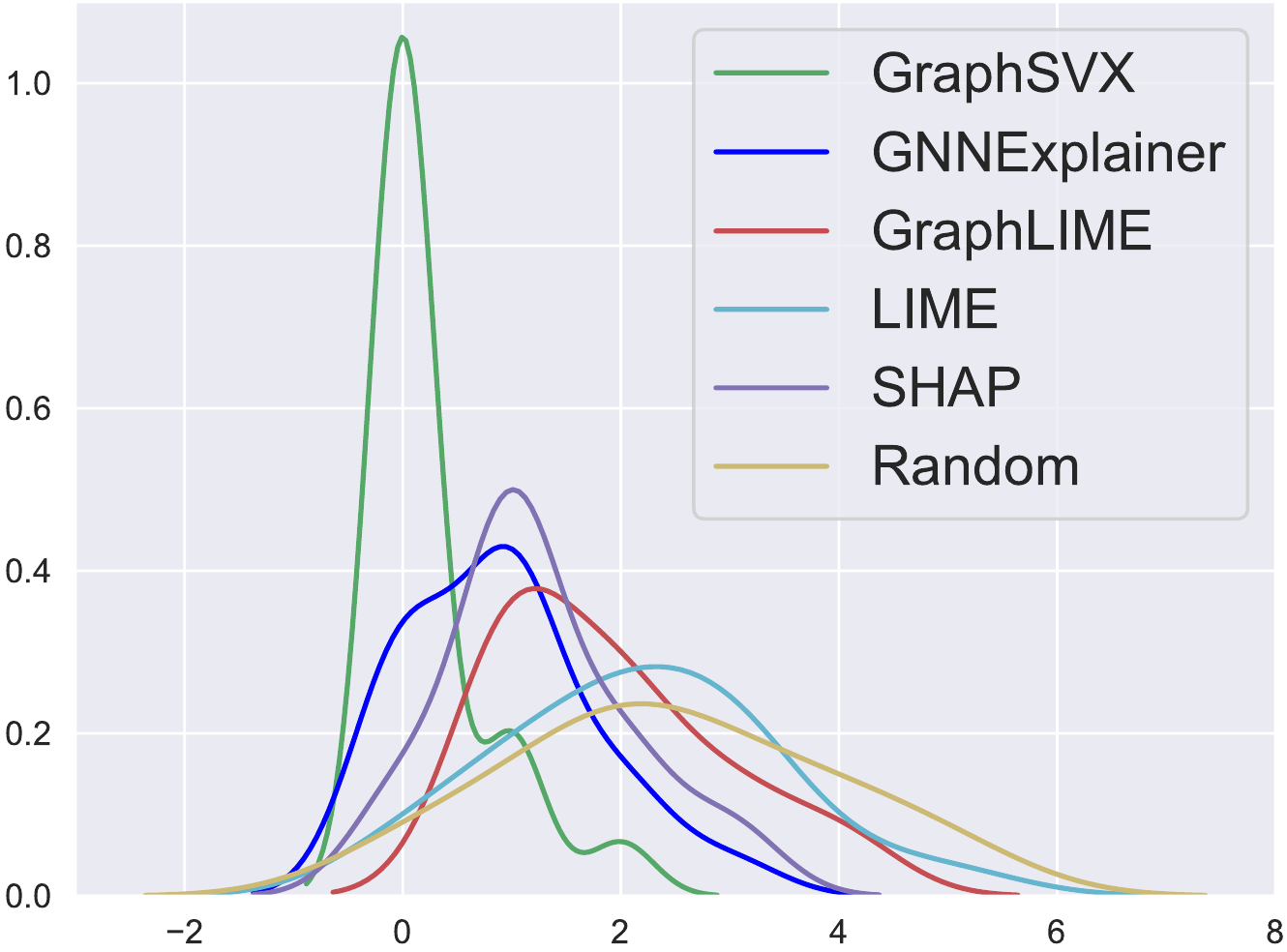}
\end{minipage}%
}% 
\vspace{-.1cm}
\subfigure[\textsl{Cora}-nodes]{
\begin{minipage}[t]{0.235\linewidth}
\centering
\includegraphics[width=\linewidth]{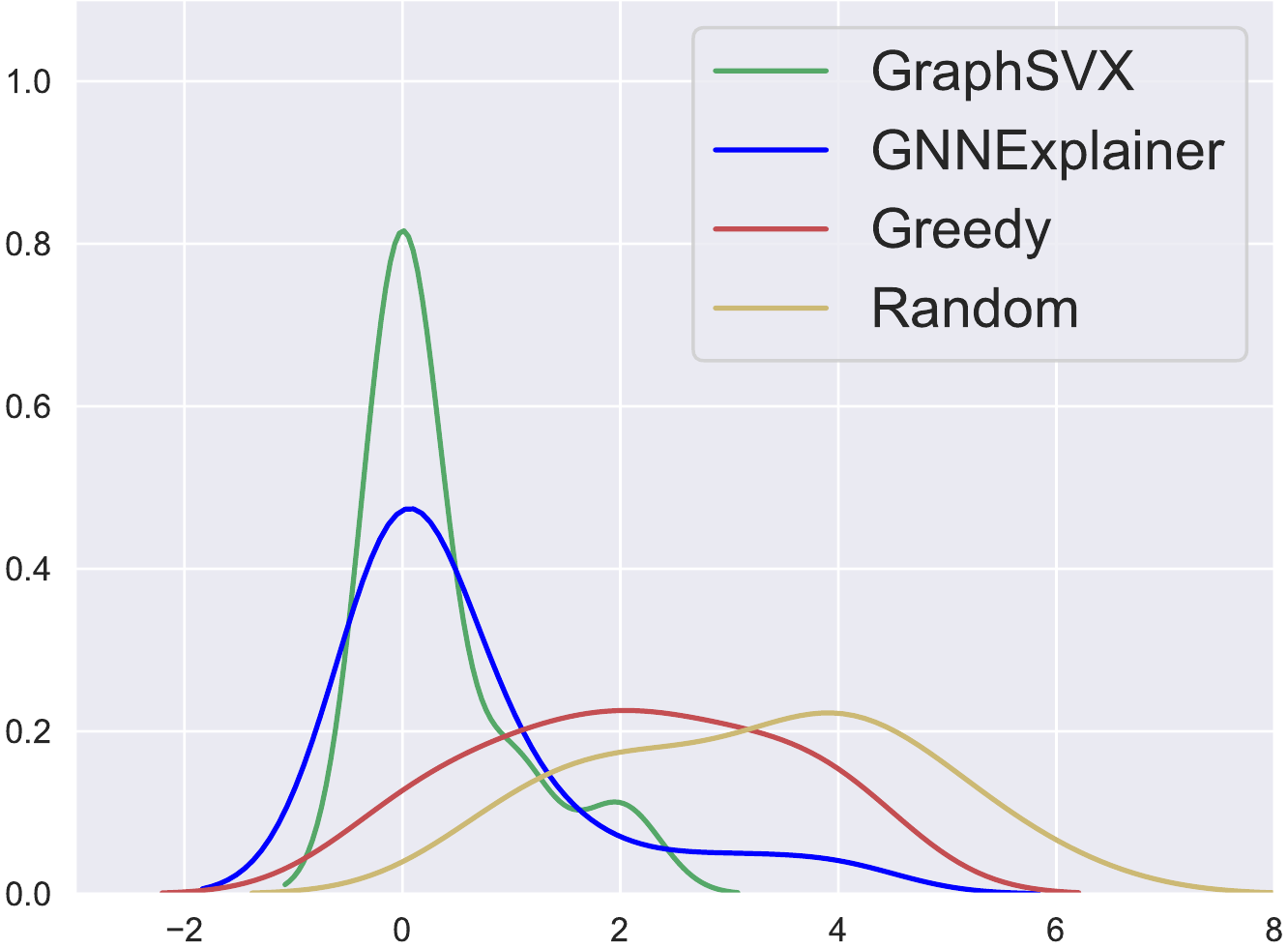}
\end{minipage}%
}%
\subfigure[\textsl{PubMed}-nodes]{
\begin{minipage}[t]{0.235\linewidth}
\centering
\includegraphics[width=\linewidth]{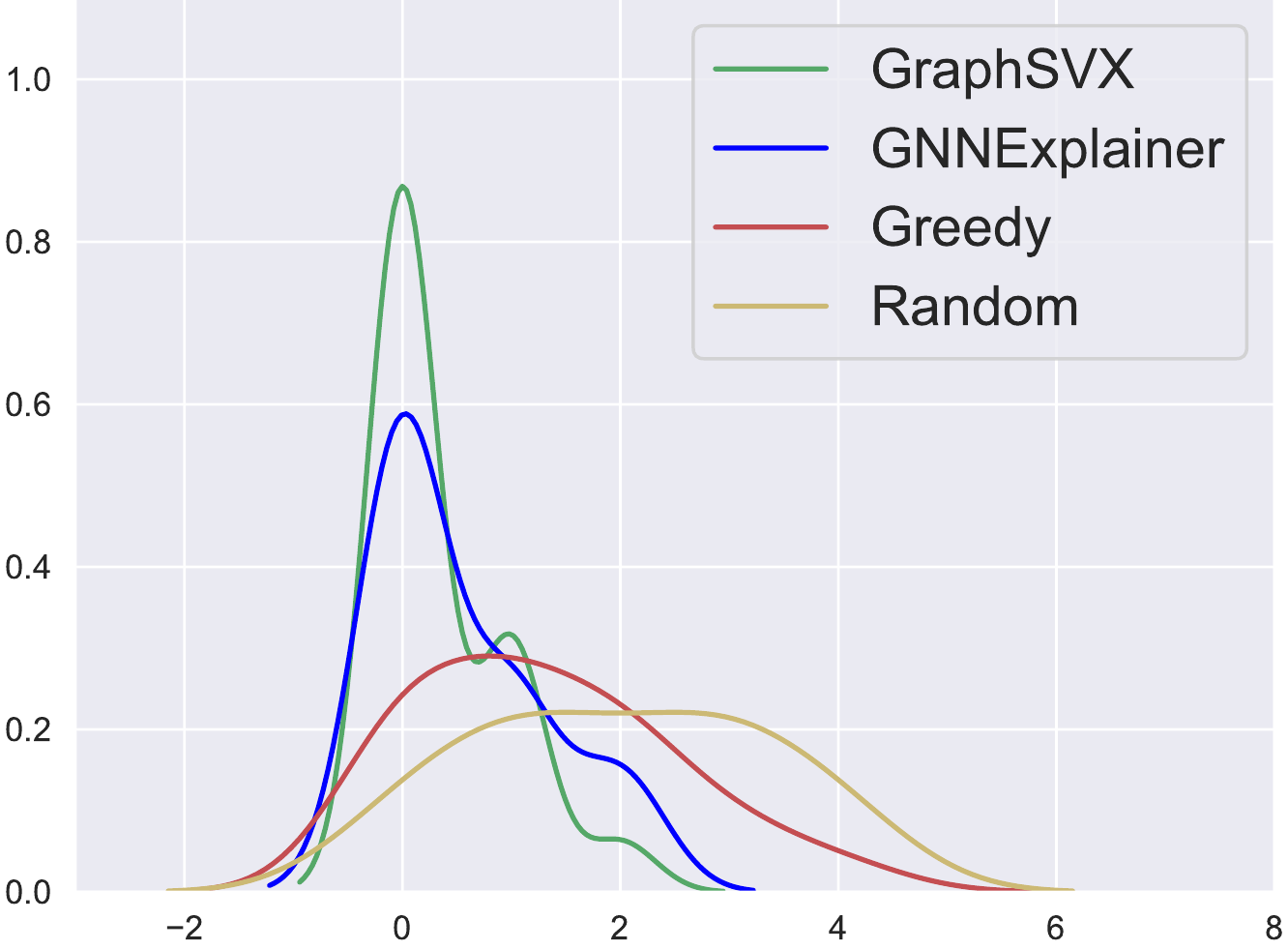}
\end{minipage}%
} 
\caption{ Frequency distributions of noisy features (a), (b) and nodes (c), (d) using a GAT model on \textsl{Cora} and \textsl{PubMed}.}
\label{KDE}
 \end{figure}
\vspace{-.05cm}

\noindent
\textbf{Noisy nodes}. We follow a similar idea for noisy neighbours instead of noisy features. Each new node's connectivity and feature vector are determined using the dataset's distribution. Only a few baselines (GNNExplainer, Greedy, Random) among the ones selected previously can be included for this task since GraphLIME, SHAP, and LIME do not provide explanations for nodes. 

As before, this evaluation builds on the assumption that a well-performing model will not consider as essential these noisy variables. We check the validity of this assumption for the GAT model by looking at its attention weights. We retrieve the average attention weight of each node across the different GAT layers and compare the one attributed to noisy nodes versus normal nodes. We expect it to be lower for noisy nodes, which proves to be true: $0.11$ vs. $0.15$.  

As shown in Fig. \ref{KDE} (c)-(d), GraphSVX also outperforms all baselines, showing nearly no noisy nodes in explanations. Nevertheless, GNNExplainer achieves almost as good performance on both datasets (and in several evaluation settings).

\label{_Eval}

\section{Conclusion}
In this paper, we have first introduced a unified framework for explaining GNNs, showing how various explainers could be expressed as instances of it. We then use this complete view to define GraphSVX, which conscientiously exploits the above pipeline to output explanations for graph topology and node features endowed with desirable theoretical and human-centric properties, eligible of a good explainer. We achieve this by defining a decomposition method that builds an explanation model on a perturbed dataset, ultimately computing the Shapley values from game theory, that we extended to graphs. Through a comprehensive evaluation, we not only achieve state-of-the-art performance on various graph and node classification tasks but also demonstrate the desirable properties of GraphSVX.

\vspace{.5cm}
\noindent \textbf{Acknowledgements.} Supported in part by ANR (French National Research Agency) under the JCJC project GraphIA (ANR-20-CE23-0009-01).

\vspace{.5cm}
\noindent \textbf{In:} The European Conference on Machine Learning and Principles and Practice of Knowledge Discovery in Databases (ECML-PKDD) 2021.

\label{_Conclu}

\bibliographystyle{splncs04}
\bibliography{main}

\newpage

\appendix

\section{Graph Neural Networks}
\label{A_gnn}

In this section, we define the two main GNNs used during the evaluation phase. Consider a graph $\mathcal{G}$ with feature matrix $\mathbf{X}$, adjacency matrix $\mathbf{A}$, and diagonal degree matrix $\mathbf{D}$. Let $N$ be the number of nodes, $C$ the number of input features, $F$ of output features, and $\mathbf{Z}$ the output. Note that, $\mathbf{Z}\in \mathbb{R}^{N\times F}, \mathbf{X} \in \mathbb{R}^{N \times C}, \mathbf{W} \in \mathbb{R}^{C \times F}$, where $\mathbf{W}$ is a weight matrix.

\subsection{Graph Convolution Networks (GCN)}
The output of one GCN \cite{kipf2016semi} layer is obtained as follows: 
\begin{align*}
    \mathbf{H}^{(l+1)} = f(\mathbf{H}^{(l)},\mathbf{A}) = \sigma\big(\mathbf{\tilde{D}}^{-1/2}\mathbf{\tilde{A}}\mathbf{\tilde{D}}^{-1/2}  \mathbf{H}^{(l)} \mathbf{W}^{(l)}\big), 
\end{align*}
with $\mathbf{H}^0=\mathbf{X}$, $\mathbf{H}^L=\mathbf{Z}$ and $\sigma$ often chosen to be the ReLU function. Also, $\mathbf{\tilde{A}} = \mathbf{A} + \mathbf{I}$ and $\mathbf{\tilde{D}} = \mathbf{D} + \mathbf{I}$ (adding self loops). Each element in $\mathbf{H}^{(l+1)}$ can thus be written as $\mathbf{h}^{(l+1)}_i = \sigma \Big(\sum_{j \in \mathcal{N}_{v_i}} \dfrac{1}{c_{ij}} \mathbf{W}^{(l)}\mathbf{h}_j^{(l)} \Big)$.

\subsection{Graph Attention Networks (GAT)}

The GAT \cite{velivckovic2017graph} model is extremely similar to the GCN model. The key difference lies in how information is aggregated from one-hop neighbour: 
\begin{align*}
    \mathbf{h}^{(l+1)}_i &= \sigma \Big(\sum_{j \in \mathcal{N}_{v_i})} \alpha_{ij}^{(l)} \mathbf{W}^{(l)} \mathbf{h}_j^{(l)} \Big), \\ 
    \text{with} \quad \alpha_{ij} &= \dfrac{\exp(\text{LeakyReLU}(\vec{a}^\top[\mathbf{W}\mathbf{h}_i||\mathbf{W} \mathbf{h}_j]))} {\sum_{k \in \mathcal{N}_{v_i}} \exp(\text{LeakyReLU}(\vec{a}^\top[\mathbf{W}\mathbf{h}_i||\mathbf{W}\mathbf{h}_k]))}.
\end{align*}
In the above, $\vec{a}$ is single feed forward neural network that concatenates both inputs and produces a unique value. 

We can have multi-head attention where the various heads utilised are either concatenated or averaged as follows: 
\begin{equation*}
    \mathbf{h}_i^{(l+1)}= \sigma \Big(\dfrac{1}{K}\sum_{k=1}^K\sum_{j \in \mathcal{N}_{v_i}}\alpha_{ij}^{(l)} \mathbf{W}_k^{(l)} \mathbf{h}_j^{(l+1)} \Big) \quad \text{or} \quad \mathbf{h}_i^{(l+1)} = ||_{k=1}^K \sigma \Big(\sum_{j \in \mathcal{N}_{v_i}}\alpha_{ij}^{(l)} \mathbf{W}_k^{(l)} \mathbf{h}_j^{(l)} \Big).
\end{equation*}

%%%%%%%%%%%%%%%%%%%%%%%%%%%%%%%%%%%%%%%%%%%%%%%%%%%%%%%%%%%%

\vspace{2mm}

\section{Extended Shapley Value Axioms}
\label{A_axioms}

Here, we provide the extension of the Shapley Value Axioms \cite{shapley1953value} to graphs. Players designate nodes of the graph (except $v$) and features of $v$; they are $F+N-1$. $\mathcal{P}$ is the power set, meaning all possible combinations of players. 
\begin{axiom}[Efficiency]
Features and nodes' contributions must add up to the difference between the original prediction and the average model prediction, i.e., $\sum_{j=1}^{F+N-1} \phi_j = f_v(\mathbf{X},\mathbf{A}) - \mathbb{E}[f_v(\mathbf{X},\mathbf{A})]$.
\end{axiom}

\begin{axiom}[Symmetry]
If~ $\forall S \in \mathcal{P}, \text{ and } k, j \notin S, \text{\textit{val}}(S\cup\{k\}) = \text{\textit{val}}(S\cup\{j\}), \text{then } \phi_k(\text{\textit{val}}) =\phi_j(\text{\textit{val}})$. 
The contribution of two players $j$ and $k$ should be identical if they contribute equally to all possible coalitions of nodes and features. 
\end{axiom}

\begin{axiom}[Dummy]
If $\forall S \in \mathcal{P}, \text{ and } j \notin S, \text{\textit{val}}(S\cup \{j\}) = \text{\textit{val}}(S), \text{ then } \phi_j(\text{\textit{val}}) = 0$.  A player $j$ which does not influence the predicted value regardless of the coalition sampled should receive a Shapley Value of $0$. 
\end{axiom}

\begin{axiom}[Additivity]
% $\phi_j(v+w) = \phi_j(v) +  \phi_j(w)$ for each characteristic functions
% $v$ and $w$.
$\text{ For all pairs of characteristic functions } v,w: \phi_j(v+w) = \phi_j(v) + \phi_j(w), \text{ where } (v+w)(S) = v(S) + w(S) \quad \text{for all } S \in \mathcal{P}$. 
% If two coalition games described by gain functions $v$ and $w$ are combined, then the distributed gains should correspond to the sum of individual gains. 
For each instance, the sum of the Shapley Values for two different prediction tasks is equal to the Shapley Value of one task whose prediction is defined from the sum of the two instances' predictions.
This last axiom constrains the value to be consistent in the space of all prediction tasks.
\end{axiom}

%%%%%%%%%%%%%%%%%%%%%%%%%%%%%%%%%%%%%%%%%%%%%%%%%%%%%%%%%%%%

\vspace{2mm}

\section{GraphSVX}
\label{A_GSVX}

In this section, we provide the pseudocode of some algorithms evoked in the paper: GraphSVX, the mask generator \textsc{Mask} and the indirect effect of nodes on prediction.

% GraphSVX 2
\begin{algorithm}[h]
    \centering
    \caption{GraphSVX}
    \label{alg:graphsvx}
\begin{small}
    \begin{algorithmic}[1]
        \STATE {\bfseries 
        Input:} $\mathcal{G} = (\mathbf{X}, \mathbf{A})$ with $N$ nodes and $F$ features. GNN model $f$. Interpretable domain $\Omega$, loss function $\mathcal{L}$ and kernel weight $\boldsymbol{\pi}$ as defined in Eq.(3). $P$: number of samples. $\textsc{Gen},\textsc{Mask}, \textsc{Expl}$ are detailed in the paper;
        \STATE $\mathcal{D} \leftarrow \{ \}$ ;
        \STATE $\mathbf{Z} \leftarrow \textsc{Mask}(\mathbf{X}, \mathbf{A})$ \tcp*{create a dataset of binary masks (feat and nodes)} 
         %\FOR{$ \mathbf{z}$ in $\mathbf{Z'}$:}
         \FOR{$ (\mathbf{M}_F,\mathbf{M}_N) $ in $\mathbf{Z}$:}
            \STATE $ \mathbf{z} \leftarrow (\mathbf{M}_F||\mathbf{M}_N)$ \tcp*{concatenation} 
            \STATE $(\mathbf{X'}, \mathbf{A'}) \leftarrow \textsc{Gen}(\mathbf{M}_F,\mathbf{M}_N,\mathbf{X}, \mathbf{A})$. \tcp*{convert masks to the original input space} 
            %\STATE $f(z) \leftarrow GNN(z)$
            \STATE $(\mathbf{X'}, \mathbf{A'}) \leftarrow \textsc{IE}(\mathbf{X'}, \mathbf{A'}, \mathbf{X}, \mathbf{A})$ \tcp*{see Algorithm \ref{alg:indirect_effect}}
            \STATE $\textbf{y} \leftarrow f(\mathbf{X'}, \mathbf{A'})$  \tcp*{new GNN prediction}
            \STATE $\mathcal{D} \leftarrow \mathcal{D} \cup \{ (\textbf{z}, \textbf{y} ) \}$ \tcp*{add sample to dataset}
        \ENDFOR
        \STATE $g \leftarrow \textsc{Expl}(\mathcal{D}, \mathcal{L}_{f, \boldsymbol{\pi}}, \Omega)$ \tcp*{learn explanation model $g$}
        %\STATE Solve $g^* = \arg \min_{g\in \Omega} \mathcal{L}_{f, \pi}(g)$ 
        % Train WLR g on D with loss L and weight \pi. 
        \STATE {\bfseries 
        Return} parameters $\boldsymbol{\phi}$ of $g$ 
    \end{algorithmic}
\end{small}
\end{algorithm}

\vspace{-2em}

\begin{algorithm}[h]
    \centering
    \caption{IE - Indirect effect}
    \label{alg:indirect_effect}
\begin{small}
    \begin{algorithmic}[1]
        \STATE {\bfseries 
        Input:} Explained node $v$, $\mathcal{N}^k_v$ set of $k$-hop neighbours of $v$, binary masks $\mathbf{z}$, $\mathcal{G}=(\mathbf{X}, \mathbf{A})$ and  $\mathcal{G'}=(\mathbf{X'}, \mathbf{A'})$;
        \STATE $\mathcal{G'_S} \leftarrow Subgraph_{\mathcal{G'}}(\mathcal{N}^k_v)$ \tcp*{k-hop subgraph of $v$ in $\mathcal{G'}$}
        
        \FOR{ $w$ in $\{u \in \mathcal{N}^k_v | z_u=1\}$}
            \IF{$\textsc{Disconnected}(w,v)$ in $\mathcal{G'_S}$}
                \STATE $\mathcal{P}_{all} \leftarrow \textsc{Dijkistra}_{\mathcal{G}}(v,w)$ \tcp*{finds shortest paths in $\mathcal{G}$ between (v,w)} 
                \STATE $\mathcal{P} \leftarrow$ \textsc{RandomSample}$(\mathcal{P}_{all}, 1)$  \tcp*{sample one path at random} 
                \FOR{$e$ in $\mathcal{P}_{edges}$}
                    \STATE $\mathbf{A}'_e = 1$ \tcp*{include in $\mathcal{G'}$ edges of the selected path} 
                \ENDFOR
                \FOR{$n$ in $\mathcal{P}_{nodes} \setminus \{v,w\}$}
                    \STATE $\mathbf{X}'_{n*} = \mathbf{X}_{v*}$ \tcp*{set new nodes of $\mathcal{G'}$ to feature values of v} 
                \ENDFOR
            \ENDIF
        \ENDFOR
        \STATE {\bfseries 
        Return} $(\mathbf{X'}, \mathbf{A'})$ 
    \end{algorithmic}
\end{small}
\end{algorithm}

\begin{algorithm}[h]
    \centering
    \caption{\textsc{Mask} - Smart sampling for nodes and features}
    \label{alg:separate_smart_sampling}
    \begin{small}
    \begin{algorithmic}[1]
       \STATE {\bfseries 
        Input:} number of samples $P$ (budget), $F$ features and $N$ neighbours, $M=F+N$, $r$ regularisation parameter for node/feature importance, $\textsc{SmarterSeparate}$, kernel weight function $\pi_v$, dataset $\mathbf{Z} \in \mathbb{R}^{P\times (F+N)}$ of samples $\mathbf{z}$;
        %\STATE $P' \leftarrow int(P\times \dfrac{F}{N+F}$
        \STATE $P' = \left\{
        \begin{array}{ll}
            \text{int}( r * P ) & \mbox{if $r$ exists} \\
            \text{int}(P \times F/M ) & \mbox{if $F=0$ or $D=0$}  \\
            \text{int}(0.5* P/2 + 0.5 * P * F/M) & \mbox{else}
        \end{array}
        \right.$  \tcp*{samples dedicated to features}
        % \STATE $P' \leftarrow int(P \times F/(N+F))$ \tcp*{number of samples dedicated to features}
        \STATE $\boldsymbol{\pi} \leftarrow \{\}$ ;
        \STATE $\mathbf{Z}[:P',F:] \leftarrow \mathbf{0}$, $\mathbf{Z}[P':,:F] \leftarrow \mathbf{1}$ ;
        \STATE $\mathbf{Z}_f \leftarrow \textsc{SmarterSeparate}(\mathbf{X}, \mathbf{A}, P', F)$  \tcp*{see Algorithm \ref{alg:mask_generator}}
        \STATE Shuffle rows of $\mathbf{Z}_f$ ;
        \STATE $\boldsymbol{\pi} \leftarrow \boldsymbol{\pi} \cup \pi_v(\mathbf{Z}_f)$ \tcp*{compute weights for each sample in $\mathbf{Z}_f$}
        \STATE $\mathbf{Z}[:P',:F] \leftarrow \mathbf{Z}_f$ ;
        \STATE $\mathbf{Z}_n \leftarrow \textsc{SmarterSeparate}(\mathbf{X}, \mathbf{A}, P-P', N)$ \tcp*{repeat process for nodes}
        \STATE Shuffle rows of $\mathbf{Z}_n$ ;
        \STATE $\boldsymbol{\pi} \leftarrow \boldsymbol{\pi} \cup \pi_v(\mathbf{Z}_n)$ ;
        \STATE $\mathbf{Z}[P':,F:] \leftarrow \mathbf{Z}_n$ ;
        \STATE {\bfseries 
        Return} $(\boldsymbol{\pi}, \mathbf{Z})$
    \end{algorithmic}
    \end{small}
\end{algorithm}

While the first two pseudocodes are well-detailed within the paper and easy to understand, we will give some more intuition regarding the mask generator, split into two pseudocodes. The first one, \textsc{Mask}, constructs the dataset of masks, separating samples where we consider nodes from those where we consider features, following from the \textit{relative efficiency} axiom. In practice, it samples $P'$ coalitions of features ($\mathbf{M}_F$) according to our sampling algorithm \textsc{SmarterSeparate} (Algorithm \ref{alg:mask_generator}) and computes the associated kernel weights of this ``incomplete" sample. This means that high weight is given to coalitions having all features but zero nodes, which was not the case beforehand. We then set to $\mathbf{0}$ the corresponding node masks $\mathbf{M}_N$ to form final samples $\mathbf{z}=(\mathbf{M}_F, \mathbf{M}_N)$, because we capture better the influence of $v$'s features on prediction when $v$ is isolated. Note that $P'$ is determined based on the proportion of nodes among all variables (nodes and features) considered. We repeat the same process for nodes, but this time, we include all features ($\mathbf{M}_F = \mathbf{1}$) in the final sample $\mathbf{z}=(\mathbf{M}_F, \mathbf{M}_N)$ as we would like to study the influence of neighbours of $v$ when $v$ is defined as it is (with all its features). Note the sum of features and nodes coefficients are fixed by the close to infinity weights enforced on the null coalition, the full coalition and the coalition of all features but zero nodes. We ultimately generate the dataset $\mathbf{Z}$ of binary concatenated masks of nodes and features, as well as their respective weights. The global process is detailed in Algorithm \ref{alg:graphsvx} (GraphSVX). \\

\begin{algorithm}[H]
    \centering
    \caption{\textsc{SmarterSeparate} - Smarter Mask generator}
    \label{alg:mask_generator}
    \begin{algorithmic}[1]
    \STATE {\bfseries Input:} Number of samples $P$, $M$ variables, maximum size $S$ of samples favoured, graph $\mathcal{G}=(\mathbf{X}, \mathbf{A})$;
    \STATE $\mathbf{Z} \leftarrow \{\} $ ;
    \STATE $P' \leftarrow \dfrac{9}{10} \times P$ ;
    \STATE $i=0, k=0$ ;
    \WHILE{$i<P' \text{ and } k\leq \min(S,M-1)$}
        \IF{ $i + 2 \times \binom{M}{k} < P'$}
            \STATE $\mathcal{C} \leftarrow \textit{ \{All\_coal\_of\_order\_k\}}$ \tcp*{set of all subsets of order k}
            \FOR{c in $\mathcal{C}$}
                \STATE $\mathbf{v} \leftarrow \mathbf{1}, \mathbf{u} \leftarrow \mathbf{0}$ ;
                \STATE $\mathbf{v}[c] \leftarrow \mathbf{0}$,  $\mathbf{u}[c] \leftarrow \mathbf{1}$  \tcp*{change value of subset's elements only}
                \STATE $\mathbf{Z} \leftarrow \mathbf{Z} \cup \{\mathbf{u}, \mathbf{v}\}$ \tcp*{add two samples to dataset}
            \ENDFOR
            \STATE $k \pluseq 1, i\pluseq \binom{M}{k}$ ;
        \ELSE  
            \STATE $\textbf{w} \in \mathbb{R}^M \leftarrow \mathbf{1}$ \tcp*{init weights}
            \STATE $\mathcal{C} \leftarrow \textit{ \{All\_coal\_of\_order\_k\}}$ ;
            \WHILE{$i < P'$}
                \STATE $\textbf{cw} \leftarrow \textsc{ComputeWeights}(\mathcal{C},\mathbf{w})$ \tcp*{def subset weight as sum of its elements' w}
                \STATE $\mathcal{C}_{max} \leftarrow \textsc{MaxWeightCoal}(\mathcal{C},\mathbf{cw})$ \tcp*{keep only coalitions with max weight}
                \STATE $c \leftarrow \textsc{RandomSample}(\mathcal{C}_{max})$ \tcp*{randomly sample one coalition}
                \STATE $p \leftarrow \textsc{Bernoulli}(0.5)$ \tcp*{p=0 or 1}
                \STATE $\mathbf{v} \leftarrow  [p.repeat(M)]$ \tcp*{vector of dim M with repeated value p}
                \STATE $v[c] \leftarrow [(1-p).repeat(len(c)]$ \tcp*{replace value of elements in c by 1-p}
                \STATE $\mathbf{Z} \leftarrow \mathbf{Z} \cup c$ ;
                \STATE $\mathbf{w}[c] \leftarrow (1+w^{-1})^{-1}$ \tcp*{update weights of coalition's elements}
                \STATE $i \pluseq 1$ ;
            \ENDWHILE
        \ENDIF
    \ENDWHILE
    \STATE $\mathbf{Z} \leftarrow \mathbf{Z} \cup \textsc{RandomCoalitions}(M, P-P')$ \tcp*{sample (P-P') random binary vectors of size M using a Bernoulli(0.5) distribution, and add them to the dataset}
    \STATE {\bfseries 
        Return} $\mathbf{Z}$ 
    \end{algorithmic}
\end{algorithm}

\textsc{SmarterSeparate} (Algorithm \ref{alg:mask_generator}), is a sampling algorithm that favours coalitions with high weight (small or large dimension) using smart space allocation. 10\% of the space is left for random coalitions. The rest is divided as follows. If there is enough space, it samples all possible combinations of variables (from 1 to $M$) of order $k$, where $k$ starts at $0$ and is incremented once they are all sampled. We create and add to the dataset two samples from each subset, one where all variables are included (1) except those in the subset (0), the other where all variables are excluded (0) except those in the subset (1). We repeat the process until there is not enough space left to sample all subsets of order k. In this case, we select iteratively samples of order $k$ ($\mathbf{1}$ with $k$ zeros, $\mathbf{0}$ with $k$ ones) until we reach the desired number, weighting the probability of each subset being drawn as the sum of its components' weights. The latter is defined such that variables that have already been sampled receive lower weights and therefore get a lower probability of being sampled again. This fosters diversity. Note that this is the most complex mask generator. They are some more simple ones that you can choose from in the code, and which are mentioned in the evaluation part below. 
% FIGURE

%%%%%%%%%%%%%%%%%%%%%%%%%%%%%%%%%%%%%%%%%%%%%%%%%%%%%%%%%%%%

\section{Proof of Theorem \ref{theorem}}
\label{A_proof}

In this section, we prove that GraphSVX computes the extended Shapley Values to graphs via the coefficients $\boldsymbol{\phi}$ of the explanation model $g$ as defined in the paper. Note that, we take some ideas from \cite{lundberg2017unified} in this proof.

Let $M=F+N$ be the total number of players (nodes and features). $\mathbf{Z} \in \mathbb{R}^{2^M \times M}$ is the matrix whose rows represent the random variables $\mathbf{z} = (\mathbf{M}_F||\mathbf{M}_N)$, that is, all possible coalitions of features of $v$ and graph nodes (but $v$). Let $S$ the set of indices $j \in \{1,\ldots,M\}$ where $\mathbf{z}_j=1$ and $s=|S|$. $\mathbf{W} \in \mathbb{R}^{2^{M} \times 2^{M}}$ is a diagonal matrix containing the kernel weights $k(M,s) = \boldsymbol{\pi}_{\mathbf{z}} = \dfrac{M-1}{M \cdot s} \cdot \binom{M-1}{s}^{-1}$, and $y_i=f_v(\mathbf{z}^{(i)})$ represents the output of the GNN model $f$ for a new instance $\mathbf{z}^{(i)}$---obtained via $\textsc{Gen}(\mathbf{z}^{(i)})$. Since $g$ is a Weighted Linear Regression, the known formula to estimate its parameters is: 
\begin{equation*}
    \mathbf{\phi} = (\mathbf{Z}^\top\mathbf{WZ})^{-1}\mathbf{Z}^\top\mathbf{Wy}.
\end{equation*}

\noindent In the above expression, the term $\mathbf{Z}^\top\mathbf{W}$ is an $M \times 2^M$ matrix whose each entry $(i,j)$ contains the sample weight $\boldsymbol{\pi}_{\mathbf{z}^{(i)}}$ if the corresponding player is present in this coalition $\mathbf{z}^{(i)}_j=1$, and $0$ otherwise. $\mathbf{Z}^\top\mathbf{WZ}$ has a more tricky interpretation. An entry $(j,j)$ in the diagonal contains the sum of all sample weights where player $j$ is present, and the $(i,j)^{th} = (j,i)^{th}$'s entry corresponds to the sum of sample weights where both player $i$ and player $j$ are present in the coalition, i.e. $z_i=z_j=1$.

% Proof formula
Remember we must have $k(M,0)=k(M,M)=\infty$, so $\mathbf{W}$ is infinity for the zero row of $\mathbf{Z}$ and its row of ones. However, if we set these infinite weights to a large positive constant $c$, (with $\mathbf{I}$ being the identity matrix and $\mathbf{J}$ the matrix of ones), then 
\begin{equation}
 (\mathbf{Z}^\top \mathbf{WZ})= \dfrac{M-1}{M} \mathbf{I} + c\mathbf{J}. 
 \label{eq:proof1}
\end{equation}

\noindent Computing the expression of $\mathbf{Z}^\top\mathbf{WZ}$, the right-hand side of the sum in the equation follows easily, as the constant weight $c$ attributed to the complete coalition (set of 1s) is added to each entry of the matrix $\mathbf{Z}^\top\mathbf{WZ}$. Let's thus study how the first component is obtained. Denoting by $E_j$ the set of samples where $z_j=1$, it comes back to showing that the difference between the $(j,j)^{th}$ entry, which is in fact the sum of the weights given to each sample in $E_j$, and the $(j,i)^{th}$ entry, which denotes the sum of the weights of samples in $E_j \cap E_i$, equals $\dfrac{M-1}{M}$. Because we are looking at the identity matrix, we can choose any $i,j$ such that $i\neq j$. We thus want to prove that each diagonal entry is equal to $\dfrac{M-1}{M}$, and we proceed as follows. 

The number of $\mathbf{z} \in E_j \setminus E_i$ is $\binom{M-2}{s-1}$ as we need to place $(s-1)$ 1s into $M-2$ free spots, since in this set, $\mathbf{z}_j=1$ and $\mathbf{z}_i=0$ by definition. It follows that $1 \leq s \leq M-1$. Thus, 
\begin{align*}
    \sum_{\mathbf{z} \in E_j \setminus E_i}^{M-1} \pi_{\mathbf{z}} &= \sum_{s=1}^{M-1} \binom{M-2}{s-1} \dfrac{M-1}{M \cdot s} \cdot \binom{M-1}{s}^{-1} \\
    &= \sum_{s=1}^{M-1} \dfrac{\binom{M-2}{s-1}(M-1)(s-1)!(M-s-1)!}{M!} \\
    &= \sum_{s=1}^{M-1} \dfrac{(M-1)\binom{M-2}{s-1}}{\binom{M-2}{s-1}(M-1)M} \\
    &= \sum_{s=1}^{M-1} \dfrac{1}{M} = \dfrac{M-1}{M}.
\end{align*}

% Proof inverse 
\vspace{2mm}
\noindent
We are, in fact, interested in the inverse of such matrix (when $c$ tends to infinity):
\begin{equation*}
    (\mathbf{Z}^\top\mathbf{WZ})^{-1} = \mathbf{I} + \dfrac{1}{M-1}(\mathbf{I}-\mathbf{J}).
\end{equation*}
To show this, we use the specific form of $(\mathbf{Z}^\top\mathbf{WZ})$ given in Eq. \eqref{eq:proof1} and $(\mathbf{Z}^\top\mathbf{WZ})^{-1} (\mathbf{Z}^\top\mathbf{WZ}) = \mathbf{I}$ to derive the expression of its inverse. In the equation below, we write $\mathbf{J}^2=M \cdot \mathbf{J}$ because $\mathbf{J}$ is the $M \times M$ matrix of all ones: 
\begin{align*}
    (a\mathbf{I}+c\mathbf{J})(b\mathbf{I}+d\mathbf{J})&=ab\mathbf{I}+ad\mathbf{J}+bc\mathbf{J}+cd\mathbf{J}^2 \\
    &=ab\mathbf{I} + \mathbf{J}(ad+bc) + Mcd\mathbf{J} \\ 
    &= ab\mathbf{I} + \mathbf{J} (ad+bc+Mcd).
\end{align*}
This equals $\mathbf{I}$ if $b=\dfrac{1}{a}$ and $ad+bc+Mcd=0$, so $d=-\dfrac{c}{a(Mc+a)}$, meaning that $(b\mathbf{I}+d\mathbf{J})$ is the inverse of $(a\mathbf{I}+c\mathbf{J})$. Using the above, we have that
\begin{align*}
    (\mathbf{Z}^\top\mathbf{WZ})^{-1} = (a\mathbf{I}+c\mathbf{J})^{-1} &= \dfrac{1}{a}\mathbf{I} - \dfrac{c}{a(Mc+a)}J \\
    &= \dfrac{1}{a}\mathbf{I} - \dfrac{1}{a(M+\dfrac{a}{c})}\mathbf{J} \\
    &\approx \dfrac{1}{a}\mathbf{I} - \dfrac{1}{aM}\mathbf{J} & \text{as } c \to \infty \\
    &= \dfrac{M}{M-1}(\mathbf{I} - \dfrac{\mathbf{J}}{M}) & \text{as } a=\dfrac{M-1}{M} \\
    &= \dfrac{M}{M-1}\mathbf{I} - \dfrac{\mathbf{J}}{M-1} \\
    &= \mathbf{I} + \dfrac{1}{M-1}(\mathbf{I}-\mathbf{J}).
\end{align*}

Coming back again to our expression of the weighted least squares, multiplying $(\mathbf{Z}^\top\mathbf{WZ})^{-1}$ by $\mathbf{Z}^\top\mathbf{W}$ creates a matrix of weights to apply to $\mathbf{y}$. Here, we consider without loss of generality the Shapley value for a single feature $j$ (any feature), written $\phi_j$ and thus only need to consider a single row of this $M \times 2^M$ matrix. This is equivalent to using only the $j^{th}$ row of $(\mathbf{Z}^\top\mathbf{WZ})^{-1}$: 
\begin{equation*}
    \phi_j = \sum_{i=1}^{2^M} (\mathbf{Z}^\top\mathbf{WZ})^{-1}_{(j,*)} \mathbf{Z}^\top\mathbf{W}_{(*,i)} y_i.
\end{equation*}

\noindent We shall first place our attention on a single component $(\mathbf{Z}^\top\mathbf{WZ})^{-1}_{(j,*)} (\mathbf{Z}^\top\mathbf{W})_{(*,i)}$ of the sum for now (for any $i$) and use the formula of $(\mathbf{Z}^\top\mathbf{WZ})^{-1}$ derived above to rewrite it. Letting $\textbf{1}_{Z_{i,j}=1}=1$ if $Z_{i,j}=1$ (feature $j$ of coalition $i$ is included) and $0$ otherwise, this yields:
\begin{align*}
    (\mathbf{Z}^\top\mathbf{WZ})^{-1}_{(j,*)} (\mathbf{Z}^\top\mathbf{W})_{(*,i)} 
    &= \Big[\textbf{1}_{Z_{i,j}=1} - \dfrac{(s_i - \textbf{1}_{Z_{i,j}=1})} {M-1}\Big] k(M,s_i) \\
    &= \dfrac{M-1}{M \cdot s_i} \cdot \binom{M-1}{s_i}^{-1} \textbf{1}_{Z_{i,j}=1} - (M \cdot s_i)^{-1} \cdot \binom{M-1}{s_i}^{-1} (s_i - \textbf{1}_{Z_{i,j}=1}) \\
    &= \dfrac{(M-1)(M-s_i-1)!(s_i-1)!}{M!}\textbf{1}_{Z_{i,j}=1} - \dfrac{(s_i - \textbf{1}_{Z_{i,j}=1}) (M-s_i-1)(s_i-1)!}{M!} \\
    &= \dfrac{(M-s_i-1)!(s_i-1)!}{M!}[(M-1)\textbf{1}_{Z_{i,j}=1} - (s_i-\textbf{1}_{Z_{i,j}=1})]
\end{align*}
When $\textbf{1}_{Z_{i,j}=1}=0$ we get 
\begin{equation*}
    \dfrac{(M-s_i-1)!(s_i-1)!}{M!}[0-s_i] = -\dfrac{(M-s_i-1)!s_i!}{M!}.
\end{equation*}
When $\textbf{1}_{Z_{i,j}=1}=1$ we get
\begin{equation*}
    \dfrac{(M-s_i-1)!(s_i-1)!}{M!}[(M-1)-(s_i-1)] = \dfrac{(M-s_i-2)!(s_i-1)!}{M!}.
\end{equation*}

So, back to our formula for the coefficient $\phi_j$, we write $y_i=f_v(S_i)$, define $\mathcal{N} = \{1,\ldots,M\}$ and use the simplified form of each component of the sum, to obtain:
\begin{align*}
    \phi_j &=  \sum_{S_i \subseteq \mathcal{N}, j \in S_i} \dfrac{(M-s_i-2)!(s_i-1)!}{M!} f_v(S_i) + \sum_{S_i \subseteq N, j \notin S_i} -\dfrac{(M-s_i-1)!s_i!}{M!} \cdot f_v(S_i) \\
    &= \sum_{S_i \subseteq \mathcal{N}, j \notin S_i} \dfrac{(M-s_i-1)!s_i!}{M!} f_v(S_i \cup \{j\}) - \sum_{S_i \subseteq N, j \notin S_i} \dfrac{(M-s_i-1)!s_i!}{M!} \cdot f_v(S_i) \\
    &= \sum_{S_i \subseteq \mathcal{N}, j \notin S_i} \dfrac{(M-s_i-1)!s_i!}{M!} \Big[f_v(S_i\cup \{j\}) - f_v(S_i)\Big].
\end{align*}

\noindent Assuming model linearity and feature independence,
\begin{equation*}
f_v(S_i\cup \{j\}) - f_v(S_i) = \mathbb{E}[f_v(\mathbf{X}, \mathbf{A}_{S_i} \cup \{j\})| \mathbf{X}_{S_i} \cup \{j\}] - \mathbb{E}[f_v(\mathbf{X}, \mathbf{A}_{S_i})| \mathbf{X}_{S_i}],
\end{equation*}
which is the  form of Shapley value $\phi_j$ in the case of graphs. This proves that GraphSVX computes the extension of  Shapley values to graphs via the coefficients of its weighted linear regression $g$. This however requires GraphSVX to use all $2^M$ samples for an exact computation. As this is computationally expensive, as we presented in the main paper, we have to resort to an approximation.

%%%%%%%%%%%%%%%%%%%%%%%%%%%%%%%%%%%%%%%%%%%%%%%%%%%%%%%%%%%

\section{Evaluation}
\label{A_eval}

In this section, we provide additional information about the evaluation phase. All experiments are conducted on a Linux machine with an Nvidia Tesla V100-PCIE-32GB model, driver version 455.32.00 and Cuda 11.1. GraphSVX is implemented using Pytorch 1.6.0.

\begin{table}[h]
    \centering
    %\normalsize
    %\resizebox{.95\columnwidth}{!}{
    \setlength{\tabcolsep}{3mm}{
        \begin{tabular}{ccccccc}%{|p{1.7cm}|p{1.7cm}|p{1.5cm}|p{1.5cm}|p{1.5cm}|}
            \toprule
            &\multicolumn{4}{c}{\textbf{Node Classification}} & \multicolumn{2}{c}{\textbf{Graph Classification}}\\
            \cmidrule(l){2-5} \cmidrule(l){6-7} 
            & \textsl{BA-Shapes} & \textsl{BA-Community} & \textsl{Tree-Cycles} & \textsl{Tree-Grid} & \textsl{BA-2motifs} & \textsl{MUTAG} \\
            \hline
            \#graphs & 1 & 1 & 1 &1 &1,000 &4,337 \\
            \#classes &4 &8 &2 &2 &2&2\\
            \#nodes & 700 &1,400 &871 &1,231 & 25,000& 131,488\\
            \#features & 10 & 10 & 10 & 10 & 10 & 14\\
            \#edges &4,110 &8,920 &1,950 & 3,410 & 51,392& 266,894\\
            \bottomrule
        \end{tabular}
    }
    \vspace{0.25em}
    \caption{Statistics of datasets with ground truth.}
    \label{tab:eval2_stats}
    \vspace{-3em}
\end{table}

\subsection{Synthetic and real datasets with ground truth}

We follow the setting exposed in \cite{ying2019gnnexplainer} and construct four kinds of node classification datasets: \textsl{BA-Shapes}, \textsl{BA-Community}, \textsl{Tree-Cycles}, \textsl{Tree-Grids} as well as two graph classification datasets: \textsl{BA-2motifs} and \textsl{MUTAG}. Table \ref{tab:eval2_stats} displays the statistics of these datasets. 

\begin{itemize}
    \item \textsl{BA-Shapes} consists of a single graph with Barabasi-Albert (BA) base graph composed of 300 nodes and 80 “house”-structured motifs. These motifs are attached to randomly selected nodes from the BA graph. In addition, $0.1N$ random edges are added to perturb the graph. Overall, nodes are assigned to one of 4 classes based on their structural role. All the ones belonging to the base graph are labelled with 0. For those in the house motifs, they are labelled with 1,2,3 if they belong respectively to the top/middle/bottom of the “house”. Features are constant across nodes.
    \item \textsl{BA-Community} dataset is a union of two BA-Shapes graphs. Node features are sampled using two Gaussian distributions, one for each BA-Shapes graph. Nodes are labelled based on their structural role and community membership, leading to 8 classes in total.
    \item \textsl{Tree-Cycles} dataset has an 8-level balanced binary tree as base graph. A set of 80 cycle motifs of size 6 are attached to randomly selected nodes from the base graph.
    \item \textsl{Tree-Grid} is constructed in the same way as \textsl{Tree-Cycles}, except that the cycle motifs are replaced by 3-by-3 grid motifs. Like the above, it is designed for node classification. 
    \item \textsl{BA-2motifs} is made for graph classification tasks and contains 800 graphs. We adopt the BA graph as base. Half of the graphs are attached with “house” motifs and the rest with five-node cycle motifs. Graphs are assigned to one of 2 classes according to the type of attached motifs.
    \item \textsl{MUTAG} is a real-life dataset for graph classification composed of 4,337 molecule graphs, each assigned to one of 2 classes given its mutagenic effect on the Gram-negative bacterium S.typhimurium ~\cite{riesen2008iam}. According to ~\cite{debnath1991structure}, carbon rings with chemical groups NO2 or NH2 are mutagenic. Since carbon rings exist in both mutagenic and nonmutagenic graphs, they are not discriminative and are thus treated as the shared base graph. NH2, NO2 are viewed as motifs for the mutagen graphs. Regarding nonmutagenic ones, there are no explicit motifs so we do not consider them.
\end{itemize}

\noindent \textbf{Experimental setup}. 
% Linux franklin 4.15.0-122-generic #124-Ubuntu SMP Thu Oct 15 13:03:05 UTC 2020 x86_64 x86_64 x86_64 GNU/Linux
% NVIDIA-SMI 455.32.00    Driver Version: 455.32.00
% Tesla V100-PCIE-32GB
% Processor: Intel(R) Xeon(R) Gold 5118 CPU @ 2.30GH
% Cuda 11.1  or 10.1
% torch==1.6.0 , torch-geometric==1.6.1
We share a single GNN model architecture across all datasets. We write as $\text{GNN}(a, b, f)$ a GNN layer with input dimension $a$, $b$ output neurons, and activation function $f$. A similar notation applies for a fully connected layer, $\text{FC}(a, b, f)$. For node classification, the structure is the following:  GNN(10, 20, ReLU)-GNN(20, 20, ReLU)-GNN(20, 20, ReLU)-FC(20, \#label, softmax). For graph classification, GNN(10, 20, ReLU)-GNN(20, 20, ReLU)-GNN(20, 20, ReLU)-Maxpooling-FC(20, \#label, softmax). Each model is trained for 1,000 epochs with Adam optimizer and initial learning rate $1.0\times 10^{-3}$. For all datasets, we use a train/validation/test split of 80/10/10\%. The accuracy metric reached on each dataset is displayed in Table~\ref{accuracy_perf_GNN_syn}. The results are good enough to make explanations of this model relevant. 

\begin{table}[h]
    %\addlinespace
    \centering
    %\normalsize
    %\resizebox{.95\columnwidth}{!}{
    \setlength{\tabcolsep}{3mm}{
        \begin{tabular}{ccccccc}%{|p{1.7cm}|p{1.7cm}|p{1.5cm}|p{1.5cm}|p{1.5cm}|}
            \toprule
            &\multicolumn{4}{c}{\textbf{Node Classification}} & \multicolumn{2}{c}{\textbf{Graph Classification}} \\
            \cmidrule(l){2-5} \cmidrule(l){6-7} 
            & \textsl{BA-Shapes} & \textsl{BA-Community} & \textsl{Tree-Cycles} & \textsl{Tree-Grid} & \textsl{BA-2motifs} & \textsl{MUTAG} \\
            \hline
            Training  & 0.97 &0.96  &0.98 &0.90  &1.00  &0.86 \\
            Validation&0.98  &0.87  &0.98  &0.89  &1.00  &0.84 \\
            Testing  &0.96  &0.88  &0.98 &0.87  &1.00 &0.85\\
            %\hline
            \bottomrule
        \end{tabular}
    }
    \vspace{0.25em}
    %\end{scriptsize}
    \caption{Accuracy performance of GNN models.}
    \label{accuracy_perf_GNN_syn}
\end{table}

% Training of Linear Regression - Pytorch settings or sklearn or WLS 

\subsection{Real-life datasets without ground truth}

We use two real-world datasets: \textsl{Cora} and \textsl{PubMed}, whose statistical details are provided in Table \ref{stats_data_eval1}. We add noise to these datasets for the next experiments and obtain six datasets in total.
\begin{itemize}
    \item \textsl{Cora} is a citation graph where nodes represent machine learning papers and edges represent citations between pairs of papers. The task involved is document classification where the goal is to categorise each paper into one of 7 categories. Each feature indicates the absence/presence of the corresponding word in its abstract. 
    \item \textsl{PubMed} is also a publication dataset where each feature indicates the TF-IDF value of the corresponding word. The task is also document classification into one of 3 classes. 
    \item \textsl{NFCora} is the \textsl{Cora} dataset where we have added 20\% of noisy features (287) to all nodes. These new noisy features are binary and have a similar distribution as existing features. We sample their value using a Bernoulli(p) distribution, where $p=0.013$. 
    \item \textsl{NFPubMed} is the \textsl{PubMed} dataset where we have added 20\% of noisy features (100) to all nodes. These new noisy features are continuous and have a similar distribution as existing features: they follow a Uniform distribution within [0,1) with probability $0.1$ and take a null value otherwise. 
    \item \textsl{NNCora} is the \textsl{Cora} dataset where we have added 20\% of noisy nodes to the graph (541). Each new noisy node is connected to an existing node with probability $0.003$, to present similar connectedness properties as existing nodes. Features of these new nodes are defined as for \textsl{NFCora}. 
    \item \textsl{NNPubMed} is the \textsl{PubMed} dataset where we have added 20\% of noisy nodes to the graph. Each new node is connected to an existing node according to a Bernoulli($0.0005$) distribution, to present similar connectedness properties as existing nodes. Features of these new nodes are defined as for \textsl{NFPubMed}. 
\end{itemize}

\begin{table}[h]
    %\addlinespace
    \centering
    \normalsize
    %\resizebox{.95\columnwidth}{!}{
    \setlength{\tabcolsep}{7mm}{
        \begin{tabular}{ccc}%c
            \toprule
            Datasets & \textsl{Cora} & \textsl{Pubmed} \\
            \hline
            \#classes & 7 & 3 \\
            \#nodes & 2,708 & 19,717 \\
            %\hline
            \#features & 1,433 & 500 \\
            %\hline
            \#edges & 5,429 & 44,338 3\\
            %\hline
            \bottomrule
        \end{tabular}
    }
    \vspace{0.25em}
    \caption{Statistics of \textsl{Cora} and \textsl{Pubmed} datasets.}
    \label{stats_data_eval1}
\end{table}

\begin{table}[h]
    %\addlinespace
    \bigskip
    \centering
    \normalsize
    %\resizebox{.95\columnwidth}{!}{
    \setlength{\tabcolsep}{3mm}{
        \begin{tabular}{ccccccc}%c
            \toprule
            Model-Data & \textsl{Cora} & \textsl{Pubmed} & \textsl{NFCora} & \textsl{NFPubMed} & \textsl{NNCora} & \textsl{NNPubMed} \\
            \hline
            GCN-Train & 0.91 & 0.85 & 0.92 & 0.85 & 0.90 & 0.73  \\
            GCN-Val & 0.88 & 0.89 & 0.88 & 0.88 & 0.86 & 0.77 \\
            GCN-Test & 0.86 & 0.88 & 0.86 & 0.86 & 0.84 & 0.77\\
            \hline
            GAT-Train & 0.78 & 0.83 & 0.79 & 0.83 & 0.74 & 0.76 \\
            GAT-Val & 0.88 & 0.88  & 0.88 & 0.88 & 0.86 & 0.85 \\
            GAT-Test & 0.86 & 0.85 &  0.86 & 0.85 & 0.86 & 0.83\\
            %\hline
            \bottomrule
        \end{tabular}
    }
    \vspace{0.25em}
    \caption{Performance of GCN/GAT models -- on original and noisy \textsl{Cora} and \textsl{PubMed} datasets. \textsl{NFCora} signifies \textsl{Cora} dataset with Noisy Features, \textsl{NNCora} points to \textsl{Cora} dataset with Noisy Nodes added. Similarly for \textsl{NFPubMed} and \textsl{NNPubMed}. }
    \label{perf_eval1}
\end{table}

\noindent \textbf{Experimental setup}. We use the same notation as before to describe GCN and GAT models' architecture, except for GAT where we add one last element representing the number of attention heads. For all models below, we use an Adam optimiser and a negative log-likelihood loss function. 

For \textsl{Cora}, we train a 2-layer GCN model GCN(1433, 16, ReLU)-GCN(16, 7, ReLU) with  50 epochs, dropout$=0.5$, learning rate (lr) $= 0.01$ and weight-decay (wd) $=5e-4$. We also train a 2-layer GAT model with structure: GAT(1433, 8, ReLU, 8)-GAT(8, 7, ReLU, 1) with 80 epochs, dropout$=0.6$, lr$=0.005$, wd$=5e-4$.

For \textsl{PubMed}, we train a 2-layer GCN model GCN(500, 16, ReLU)-GCN(16, 3, ReLU) with dropout$=0.5$, $150$ epochs, lr$=0.01$, wd $= 5e-4$. We also train a 2-layer GAT model with structure: GAT(500, 8, ReLU, 8)-GAT(8, 3, ReLU, 8) with 120 epochs, dropout$=0.6$, lr$=0.005$ and wd$=5e-4$. \\

The results provided in the paper are obtained by constructing a Kernel Density Estimator (KDE) plot from the distribution of noisy nodes/features that are included in top-k explanations for $50$ different test samples (50 node predictions explained), using the above trained GAT model. 
% Give details number of samples for each dataset and number of neighbours 

\subsection{Ablation study and hyperparameters tuning}

During the evaluation process, we have tested the impact of different hyperparameters as well as different versions of GraphSVX so as to validate the improvements we thought of. Here is a non-exhaustive list of the conclusions we were able to draw for this analysis:

\begin{itemize}
    
    \item We show empirically that \textbf{reducing the number of features and nodes} considered from $F + N$ (all) to $D + B$ ($D$: $k$-hop neighbours; $B$: features not within a confidence interval around the mean value or $0$) is extremely significant. On average, features that are removed from consideration occupy $5\%$ of the explainable part $f_v(\mathbf{X},\mathbf{A}) - \mathbb{E}[f_v(\mathbf{X},\mathbf{A})]$ when considered. Similarly, all nodes that are removed occupy 8\% of the explainable part. 
    
    \item Capturing the \textbf{indirect effect} of nodes in \textsc{Gen()} augments performance by 19\% (on average) on \textsl{TreeCycles} and \textsl{TreeGrid} datasets, where higher-order relations matter more and more coalitions are sampled. For \textsl{BA-Shapes} and \textsl{BA-Community}, where most nodes in the ground truth are 1-hop neighbours, higher order relation are less important and adding indirect effect has no effect on performance. 

    \item The extension involving considering the influence of a certain \textbf{ feature from the subgraph perspective} instead of $v$'s feature value is highly relevant, especially for our evaluation framework as we add noisy features using Gaussian distributions on all nodes. It is thus easier to spot noisy features on \textsl{Cora} and \textsl{PubMed}. On \textsl{BA-Community}, we demonstrated its higher ability to capture global feature importance for similar reasons (important features are defined using the whole graph and not single nodes), as it achieved $100\%$ accuracy for essential features against $65\%$ for the standard method.
    
    \item Enforcing a high (close to infinity) weight to the masks $\mathbf{1}$ and $\mathbf{0}$ to satisfy the \textit{efficiency} axiom is not  necessary. Identical explanations are reached most of the time when these coalitions are not sampled. 
    
    \item Weighted Least Squares (WLS) and Weighted Linear Regression (WLR) yield equivalent results. In different terms, the learning method in the \textbf{explanation generator} does not impact much the results, although WLR yields slightly better ones. Furthermore, weighted Lasso is not extremely useful for synthetic datasets, rather small and sparse. 
    % Cora results
    
    \item The \textbf{base value}, or constant of our explanation model $g$, approximates the expected model prediction, as stated in the paper. For Cora, the class repartition is the following: $[0.13, 0.09, 0.15, 0.29, 0.15, 0.11, 0.07]$, for the seven targets. The base value for all classes, obtained via $g$ are: $[0.12, 0.10, 0.15, 0.27, 0.15, 0.12, 0.08]$, which is a great approximation. 
    
    \item \textbf{GAT} is more effective than \textbf{GCN} for our evaluation on real-world datasets without a ground truth, and yields better results for most baselines. 
    
    \item The \textbf{seed} chosen to generate pseudo-random numbers and thus be able to replicate identical experiments impact significantly the results for methods with high variability when a small test sample is chosen. We choose a large test sample to counter this. 
    
    \item We observe in Table \ref{tab:MASK_ablation} that our \textsc{SmarterSeparate} \textbf{mask generator} outperforms sampling baselines on all tasks - for a same number of samples $P << 2^{B+D} $ (except \textsc{All} which uses $\rightarrow 2^{B+D}$).  \textsc{All} approximately samples every possible combinations of nodes and features. \textsc{Random} simply chooses randomly binary masks. \textsc{Smart} samples in priority masks with a high weight (wrt $g)$. \textsc{SmartSeparate} additionally separates the effect of nodes and features. Finally, \textsc{SmarterSeparate} is the current approach, described in the paper. It uses a smart space allocation algorithm on top of \textsc{SmartSeparate}.
    
\end{itemize}

\begin{table}[H]
    \centering
    \begin{tabular}{ccccc}
        \toprule
         $\textsc{Mask}$ & \textsl{BA-Shapes} & \textsl{BA-Community} & \textsl{Tree-Cycles} & \textsl{Tree-Grid}   \\
         \hline
         \textsc{SmarterSeparate} & 0.99 & 0.93 & 0.97 & 0.93 \\ 
         \textsc{SmartSeparate} & 0.94 & 0.92 & 0.85 & 0.83 \\ 
         \textsc{Smart} & 0.93 & 0.91 & 0.81 & 0.79 \\ 
         \textsc{Random} & 0.84 & 0.75 & 0.74 & 0.64 \\ 
         \textsc{All} & 0.80 & 0.79 & 0.87 & 0.77 \\ 
         \bottomrule
    \end{tabular}
    \vspace{0.25em}
    \caption{ \textsc{Random} is improved by \textsc{Smart}, so sampling masks of high weight (few or many variables) is relevant. \textsc{Smart} is itself improved by \textsc{SmartSeparate}, which shows that separating nodes and features in coalitions increases performance for an identical number of coalitions $P$---in fact reduces complexity. Finally, \textsc{Smarter} is improved by \textsc{SmarterSeparate}, which justify the efficiency of the smart space allocation algorithm.}
    \label{tab:MASK_ablation}
    %\vspace{-5em}
\end{table}

\noindent Alongside with the Table \ref{tab:MASK_ablation}, we remark that increasing $P$ \textbf{(number of samples)} and $S$ \textbf{(maximum order of coalitions favoured)} is relevant until a certain point, where performance starts to slightly decrease and often converges to \textsc{All} coalition results. For example, on \textsl{Tree-Cycles}, the performance of default GraphSVX for $500$ samples is $0.83$, for $1,500$ samples $0.95$ and for $3,000$ samples $0.88$; while GraphSVX with \textsc{All} method yields $0.87$. For simple datasets (\textsl{BA-Shapes}, \textsl{BA-Community}), few samples is enough to get good accuracy, meaning when looking only at individual effect ($S=1$). Increasing $S$ and $num$\_$samples$ does not change performance, it only increases computational time. Nevertheless, on slightly more complex datasets (\textsl{Tree-Cycles}, \textsl{Tree-Grid}), optimal performance is reached when considering all coalitions of higher-order ($S=4$), which require more coalitions to be sampled. \\

To continue with \textbf{parameter sensitivity}, when testing robustness to noise, the proportion of noisy features and nodes introduced, as well as the connectivity of the nodes, the distribution of the features and the number/proportion of most important features we look at all have a significant impact on results. Nervertheless, this impact is consistent across all baselines. We therefore chose a configuration that appeared relevant, with enough noise introduced to differentiate between good and bad explainers.  \\

\noindent
\textbf{Running times} for explanation of a single node are displayed in Table \ref{tab:time} for all synthetic node classification datasets. Comparing running times with baselines \cite{luo2020parameterized, ying2019gnnexplainer}, GraphSVX is still slower than PGExplainer, which is very scalable, but matches GNNExplainer thanks to our efficient approximation, especially when the data is relatively sparse (e.g., \textsl{BA-Shapes} or \textsl{Tree-Cycles}). With a 2-hop subgraph approximation \ref{tab:time2}, it is faster than GNNExplainer while maintaining state-of-the-art performance. \\
% Considering all coalitions is untractable most of the time (2^{B+D})- cannot justify approximation.

\begin{table}[h]
    \centering
    \vspace{-1em}
    \begin{tabular}{cccccc}
        \toprule
          & & \textsl{BA-Shapes} & \textsl{BA-Community} & \textsl{Tree-Cycles} & \textsl{Tree-Grid}  \\
          \midrule
         %\hline
         GraphSVX: & Time (s) & 2.12 & 8.31 & 3.65 & 6.22  \\
         & Samples & 400 & 800 & 1,400 & 1,500  \\
         %\hline
         GNNExplainer: & Time (s) & 6.63 & 7.08 & 6.89 & 7.31 \\
         \bottomrule
    \end{tabular}
    \vspace{0.25em}
    \caption{Average running time for the explanation of a single node with GraphSVX (3-hop subgraph approximation) and GNNExplainer.}
    \label{tab:time2}
\end{table}

\begin{table}[h]
    \centering
    \vspace{-2em}
    \begin{tabular}{ccccc}
        \toprule
         & \textsl{BA-Shapes} & \textsl{BA-Community} & \textsl{Tree-Cycles} & \textsl{Tree-Grid} \\
         \midrule
         % \hline
         Time (s) & 0.30 & 1.09 & 1.21 & 1.87 \\
         Samples & 65 & 100 & 300 & 400 \\
         \bottomrule
    \end{tabular}
    \caption{Average running time for the  explanations of a single node with GraphSVX (2-hop subgraph).}
    \label{tab:time}
\end{table}

%%%%%%%%%%%%%%%%%%%%%%%%%%%%%%%%%%%%%%%%%%%%%%%%%%%%%%%%%%%%%%%%%

\vspace{2mm}

\section{Properties}
\label{A_prop}

\subsection{Desirable properties of explanations from a machine learning perspective}

% Add references to each property description - See RW paper. 

\begin{itemize}
    \item \textbf{Accuracy} and \textbf{Fidelity}: How relevant is an explanation?
    And for an explanation model, how well does it approximate the prediction of the black box model? There must be a ground truth explanation or a human judge to assess the accuracy of an explainer \cite{burkart2021survey}. 
    \item \textbf{Robustness}: How different are explanations for similar models/instances? It encapsulates consistency, stability, and resistance to noise. High robustness means that slight variations in the features of an instance, or in the model functioning, do not substantially change the explanation. Sampling or random initialisation of mask generator goes usually against robustness \cite{molnar2020interpretable}. 
    \item \textbf{Certainty}: Does the explanation reflect the certainty of the machine learning model? In different terms, does the explanation indicate the confidence of the model for the explained instance prediction? In the paper, we refer to it as \textbf{decomposability}---the most common way to reflect certainty
    \cite{robnik2018perturbation}. 
    \item \textbf{Meaningfulness}: How well does the explanation reflect the importance of features/nodes? Here, meaningful explanations means ``with a nice interpretation/signification'' \cite{yuan2020explainability}.
    \item \textbf{Representativeness} (or \textbf{Global}): How general is an explanation? People like explanations that cover many instances or the general functioning of the model. This global scope is concerned with overall actions and usually provides a pattern that the prediction model discovers \cite{miller2019explanation}.
    \item \textbf{Comprehensibility}: How well do humans understand explanations? This property comprises several aspects that are further analysed below \cite{molnar2020interpretable}. 
\end{itemize}

\subsection{Human-centric explanations}

\begin{itemize}
    \item \textbf{Contrastive}: Humans usually do not ask why a certain prediction was made, but why it was made instead of another prediction. In different terms, humans want to compare it against another instance's prediction (could be an artificial one) \cite{lipton1990contrastive}. 
    \item \textbf{Selective}: People do not expect explanations to cover the actual and complete list of causes of an event. They are used to being given one or two key causes as the explanation \cite{ustun2014methods}. 
    \item \textbf{Social}: Explanations are part of an interaction between the explainer and the receiver. The social context determines the content and nature of explanations; they should be targeted to the audience \cite{miller2017explainable}.
    % \item \textbf{Confirmation bias} Good explanations are consistent with prior beliefs of the explainee. Humans tend to ignore information that is inconsistent with their prior beliefs.
    % \item \textbf{Focus on the abnormal}: People focus more on abnormal causes to explain events. These causes had a small probability but nevertheless happened. This property is difficult to measure and thus ignored. 
    % \item \textbf{Causality}. Good explanations imply causality. This one is not part of the focus of the paper. 
\end{itemize}

\subsection{Explainers' properties}

\begin{table*}[h]
    \vspace{-1em}
    \centering
    \normalsize
    \begin{tabular}{lcccr}
    %\resizebox{.95\columnwidth}{!}{
    \setlength{\tabcolsep}{2mm}
    \begin{tabular}{lcccccccc}
        \toprule
        Method & Task & Target & \textit{Decomposable} & \textit{Global} & \textit{Robust} & \textit{Meaningful} & \textit{Human-centric} \\
        \midrule
        GNNExplainer & GC/NC & E/F & $\times$ & \checkmark & $\times$ & $\times$ & \checkmark \\
        %\hline
        PGExplainer & GC/NC & E & $\times$ & \checkmark & $\times$ & $\times$ & \checkmark \\
        %\hline
        GraphLIME & NC & F & $\times$ & $\times$ & $\times$ & \checkmark & \checkmark \\
        %\hline
        PGM-Explainer & GC/NC & N & $\times$ & $\times$ & $\times$ & $\times$ & \checkmark \checkmark \\
        %\hline
        XGNN & GC & N & $\times$ & \checkmark & $\times$ & $\times$ & \checkmark \\
        \midrule
        GraphSVX & GC/NC & N/F & \checkmark & \checkmark & \checkmark & \checkmark & \checkmark \checkmark \checkmark  \\
        \bottomrule
        \end{tabular}
    \end{tabular}
    \vspace{0.25em}
    \caption{Comparison of explanation methods for GNNs. GC/NC denote graph and node classification tasks. Target explanation points to features (F), edges (E) or nodes (N). \textit{Human-centric} is decomposed into social, contrastive and selective---and one checkmark is attributed for each property that holds. We regroup under \textit{Robust}: Consistent, Stable and robust to noise.}
    \label{prop_baselines}
\end{table*}

In Table \ref{prop_baselines}, we show how explainers belonging to the unified framework proposed in the paper satisfy these properties. Below, we develop why/how GraphSVX (and the baselines) satisfy them. \\

\noindent \textbf{Fidel and accurate}. GraphSVX achieves state-of-the-art results on all but one task, which proves its high accuracy. It is also fidel to the GNN model locally, as the explanation model $g$ reaches $R^2 > 0.90$ for all tasks. \\

\noindent \textbf{Decomposable}. GraphSVX is the only fairly distributed method. In practice, we verify that $\sum_i \phi_i = f_v(\mathbf{X},\mathbf{A})$ for all predictions, which is naturally enforced by the high weight given to both null and full coalitions. \\

\noindent \textbf{Meaningful}. Along with GraphLIME, GraphSVX is the only meaningful method---it outputs the ``fair'' contribution of each feature/node towards the prediction (with respect to the average prediction). Overall, it provides more information than all methods, including GraphLIME. For instance, on synthetic datasets, the importance granted to nodes belonging to a motif achieves 85\% of the explainable part $f_v(X,A)-\mathbb{E}[f_v(X,A)]$ for \textsl{Tree-Grid} and \textsl{Tree-Cycle}, while it represents only 20\% on \textsl{BA-Community}, where node features and the relation to other motifs (inside another community) are also essential for prediction. This provides key information to an end-user. Note that, we consider GraphLIME as meaningful although the coefficients themselves do not have proper signification and only account for feature importance. None the less, they do give an idea of the influence of a feature in prediction. Regarding other baselines, PGExplainer and GNNExplainer output probability scores, PGM-Explainer a bayesian network, and XGNN a subgraph without further information. \\

\noindent \textbf{Global}. Similarly to GNNExplainer, GraphSVX is endowed with an extension that enables it to explain a subset of nodes together, without aggregating local explanations. Furthermore, it can study feature importance in the subgraph of the node being explained instead of the node itself, which also goes towards a more global comprehension. Two other explainers provide even more general explanations: XGNN, a true model-level explanation method and PGExplainer, which provides collective and inductive explanations. \\

\noindent \textbf{Robust}. GraphSVX is by definition consistent and stable. In addition to showing its robustness to noise in the experimental evaluation section, we further investigate those two properties. We assess stability by computing variance statistics on synthetic datasets explanations, for nodes occupying the same role in a motif. Explanations are very stable across all datasets, with a variance in accuracy always inferior to $0.1$. Consistency is less trivial to examine. We compare the explanations obtained for the GCN and GAT models on \textsl{Cora}, and compute the intersection over union (IOU) score ($0.76$) between the two sets of top-$20\%$ (or top $10$) most important features and nodes. The core differences in behaviour and performance between the GCN and GAT models prevents us from obtaining a better result. Consistency is tricky to assess for every explainer. We justify theoretically why baseline explainers are not considered as robust (w.r.t. to GraphSVX). GraphLIME and PGExplainer use sampling without any approximation guarantees. XGNN uses a reinforcement learning approach, a candidate node set and several deep learning models to produce subgraph explanations. PGExplainer initialises randomly the parameters of the mask generator and uses intermediate model representations, which could respectively lead to local optimum and differences across similar models. Similarly for GNNExplainer but in a smaller extent since it does not use intermediate representations. \\

\noindent \textbf{Human-centric}. Let's split this part into different sub-properties. (1) \textbf{Selective}. Instead of fitting a WLR explanation model $g$, we fit instead a weighted Lasso regression, which produces sparser results, more easily comprehensible by humans \cite{ustun2014methods}. (2) \textbf{Contrastive}. As detailed in the paper, we extend GraphSVX to enable explaining an instance with respect to another instance. (3) \textbf{Social}. Explanations should be adapted to the target audience. Here, we offer complete explanations appropriate for domain experts, which answers the recent European regulations concerning the legal right to explanation: GDPR \cite{goodman2016eu}. We also provide for non-domain expert selective explanations with intuitive visualisations. Furthermore, the pipeline is flexible. We can fully regulate the amount of importance granted to nodes and features, either via shutting down completely the importance given to either category, or by rescaling it:
\begin{equation*}
    \boldsymbol{\phi}_{\text{node}} = \boldsymbol{\phi}_{\text{node}} \cdot \dfrac{\alpha * e}{\sum(\boldsymbol{\phi}_{\text{node}})} \quad \text{and} \quad
    \boldsymbol{\phi}_{\text{feat}} = \boldsymbol{\phi}_{\text{feat}} \cdot \dfrac{ (1-\alpha) * e}{\sum(\boldsymbol{\phi}_{\text{feat}})}, 
\end{equation*}
with $e = (f_v(\mathbf{X}, \mathbf{A}) - \mathbb{E}[f_v(\mathbf{X},\mathbf{A}))$ and $\alpha \in [0,1]$ the regularisation coefficient. 
In the above table, this leads to $3$ checkmarks out of $3$ for GraphSVX.  PGM-Explainer is selective and social, as it produces a causal graph with feature selection, but is not contrastive. GNNExplainer, XGNN, GraphLIME and PGExplainer are simply selective.  
% (1) \textbf{Confirmation bias}. Explanations are consistent with prior beliefs of the explainee \cite{nickerson1998confirmation}. This last bit is very difficult to embed in an explainer. We more or less implicitly enforce it via the monotonicity constraint imposed by the weighted linear regression, which keeps the relation between input features and output monotone. 

%\bibliographystyle{plain} % We choose the "plain" reference style
%\bibliography{main}

\end{document}